\documentclass{article}
\usepackage{iclr2027_conference,times}
\usepackage[T1]{fontenc}
\usepackage{amsmath,amssymb,amsthm}
\usepackage{graphicx}
\usepackage{wrapfig}
\usepackage{needspace}
\usepackage{booktabs}
\usepackage{multirow}
\usepackage{placeins}
\usepackage{subcaption}
\usepackage{newfloat}
\DeclareFloatingEnvironment[name=Table,placement=htbp]{supplementtable}

\usepackage{algorithm}
\usepackage{algpseudocode}
\usepackage{xspace}
\usepackage{xcolor}
\usepackage{microtype}
\usepackage{url}
\usepackage{tikz}
\usetikzlibrary{arrows.meta,positioning,calc}
\usepackage{pgfplots}
\pgfplotsset{compat=1.18}
\definecolor{FigInk}{HTML}{1F2933}
\definecolor{FigMain}{HTML}{2F6DB5}
\colorlet{FigMainText}{FigMain!75!black}
\definecolor{FigBase}{HTML}{8E98A4}
\definecolor{FigMainLight}{HTML}{D9E7F7}
\definecolor{FigNeutral}{HTML}{F4F6F8}
\definecolor{FigOutline}{HTML}{B6BFCA}
\definecolor{FigGrid}{HTML}{E6E9ED}
\definecolor{FigAxis}{HTML}{9AA3AE}
\definecolor{FigHeatLo}{HTML}{EEF4FB}
\definecolor{FigHeatHi}{HTML}{2F6DB5}
\pgfplotsset{every axis/.append style={
  axis lines=left,axis line style={FigAxis,line width=0.5pt},
  tick style={FigAxis},tick label style={font=\small,text=FigInk},
  label style={font=\small,text=FigInk},
  ymajorgrids,grid style={FigGrid,line width=0.4pt},
  error bars/y dir=both,error bars/y explicit,
  error bars/error bar style={line width=0.55pt},
  error bars/error mark options={rotate=90,mark size=1.8pt,line width=0.55pt},
  legend style={draw=none,fill=none,font=\small,legend cell align=left},
  clip=false}}
\tikzset{
  ourscurve/.style={color=FigMain,mark=*,mark size=2.2pt,mark options={solid,fill=FigMain,draw=FigMain},line width=1.1pt},
  basecurve/.style={color=FigBase,mark=square*,mark size=2.1pt,mark options={solid,fill=FigBase,draw=FigBase},line width=1.1pt},
  paneltitle/.style={font=\small\bfseries,text=FigInk,anchor=south},
  keptcell/.style={fill=FigMainLight,draw=FigMain,text=black},
}
\usepackage[bookmarksnumbered=true,colorlinks=true,linkcolor=black,citecolor=blue,urlcolor=black]{hyperref}

\DeclareRobustCommand{\paperref}[2]{#1~\ref{#2}}
\DeclareRobustCommand{\papereqref}[2]{#1~\eqref{#2}}
\DeclareRobustCommand{\papernumref}[1]{\ref{#1}}

\newcommand{\method}{Loop Dropout\xspace}
\newcommand{\E}{\mathbb{E}}
\newcommand{\R}{\mathbb{R}}
\newcommand{\Bernoulli}{\operatorname{Bernoulli}}

\newcommand{\CE}{\operatorname{CE}}
\newcommand{\tr}{\operatorname{tr}}
\newcommand{\ones}{\mathbf{1}}
\newtheorem{lemma}{Lemma}

\iclrfinalcopy
\renewcommand{\headrulewidth}{0pt}

\definecolor{CodeLink}{HTML}{24786D}
\AddToHook{env/abstract/end}{%
  \par\smallskip\noindent
  \begingroup
  \hypersetup{urlcolor=CodeLink}%
  \small\sffamily
  \textbf{Code:}\enspace
  \href{https://github.com/NUS-HPC-AI-Lab/loop-dropout}{github.com/NUS-HPC-AI-Lab/loop-dropout}%
  \par\endgroup
}

\author{%
\textbf{Zirui Zhu}$^{1\,\dagger}$\quad \textbf{Hailun Xu}$^{2}$\quad \textbf{Xuanlei Zhao}$^{1,2}$\quad \textbf{Yong Liu}$^{3}$\\[4pt]
\textbf{Yingxuan Ren}$^{1,2}$\quad \textbf{Kanchan Sarkar}$^{2\,\dagger}$\quad \textbf{Kun Xu}$^{2\,\dagger}$\quad \textbf{Yang You}$^{1\,\dagger}$\\[7pt]
{\small $^{1}$National University of Singapore\quad $^{2}$TikTok\quad $^{3}$Georgia Institute of Technology}\\[4pt]
{\small $^{\dagger}$\,Corresponding Authors}\\[3pt]
{\small\ttfamily \{\href{mailto:zirui@comp.nus.edu.sg}{zirui}, \href{mailto:youy@comp.nus.edu.sg}{youy}\}@comp.nus.edu.sg}\\[2pt]
{\small\ttfamily \{\href{mailto:kanchan.sarkar@tiktok.com}{kanchan.sarkar}, \href{mailto:daniel.chen28@tiktok.com}{daniel.chen28}\}@tiktok.com}
}

\makeatletter
\renewcommand{\@maketitle}{%
  \vbox{\hsize\textwidth
    \centering
    {\LARGE\scshape\@title\par}
    \vskip 14pt
    {\normalsize\@author\par}
    \vskip 0.3in minus 0.1in
  }%
}
\makeatother

\hypersetup{
  pdfauthor={Zirui Zhu, Hailun Xu, Xuanlei Zhao, Yong Liu, Yingxuan Ren, Kanchan Sarkar, Kun Xu, Yang You}
}
\title{Loop Dropout: Regularizing Shared Updates in Looped Language Models}
\hypersetup{
  pdftitle={Loop Dropout: Regularizing Shared Updates in Looped Language Models},
  pdfsubject={Fine-tuning shared updates in looped language models},
  pdfkeywords={Looped Language Models, Parameter-Efficient Fine-Tuning, Low-Rank Adaptation}
}

\begin{document}
\maketitle

\begin{abstract}
Looped language models separate computational depth from parameter count by repeatedly applying the same transformer block.
Adapting these models requires a shared update that remains effective as hidden states evolve throughout the recurrent computation.
Our empirical analysis reveals a pronounced late-loop bias in standard low-rank adaptation (LoRA): the shared update provides limited adaptation at early loop positions.
This imbalance motivates training shared updates under varying combinations of their applications.
Randomly omitting adapter applications alone, however, does not improve task performance; it reduces expected update strength during training while leaving inference unchanged.
We introduce \method, which couples stochastic masking of adapter applications with inverse-survival rescaling to preserve expected update strength and promote effective adaptation across loops.
Extensive experiments demonstrate improved mathematical reasoning across model sizes, adapter ranks and training recipes, with benefits extending to general instruction tuning and code generation.
\method outperforms existing LoRA variants and adapter regularizers, while further analysis shows stronger early-loop adaptation.
Every backbone loop remains active, and inference applies the adapter at all loops using standard LoRA without additional trainable parameters or inference computation.
%

\end{abstract}

\section{Introduction}
\label{sec:introduction}

Looped language models separate computational depth from parameter count by repeatedly applying the same transformer block~\citep{dehghani2019universal,geiping2025huginn,zhu2025ouro}.
This reuse allows compact models to perform multiple stages of computation and match substantially larger standard language models~\citep{zhu2025ouro}.
Adapting these models requires a shared update that remains effective as hidden states evolve throughout the recurrent computation.

Low-rank adaptation (LoRA) learns a compact update to frozen pretrained weights~\citep{hu2022lora}, with variants modifying its parameterization, optimization and regularization~\citep{hayou2024loraplus,kalajdzievski2023rslora,liu2024dora,lin2024loradropout}.
In a looped model, the same update acts on different hidden states at different distances from the final prediction.
Standard LoRA fine-tuning activates all applications together and optimizes their combined effect through the final loss.
Step-resolved data attribution shows that training examples influence the loops unevenly~\citep{kaissis2026sdi}, motivating an empirical analysis of how effectively the learned update works across the recurrence.

\textbf{Our empirical analysis reveals a pronounced late-loop bias in shared adaptation.}
\paperref{Figure}{fig:teaser_activation} evaluates a trained adapter at each loop position in turn, with the adapter active only at that position and the backbone running all four loops.
For Ouro-1.4B fine-tuned on GSM8K with standard LoRA, applying the update at the first loop reduces loss by only 12\%, compared with 32\% at the fourth loop, relative to the same frozen model.
The same late-loop bias persists across training configurations.
Sharing adapter parameters therefore does not make their applications equally effective at different loop positions.
This imbalance motivates learning an update that supports effective adaptation throughout the recurrent computation.

\begin{figure}[t]
\centering
\begingroup
\input{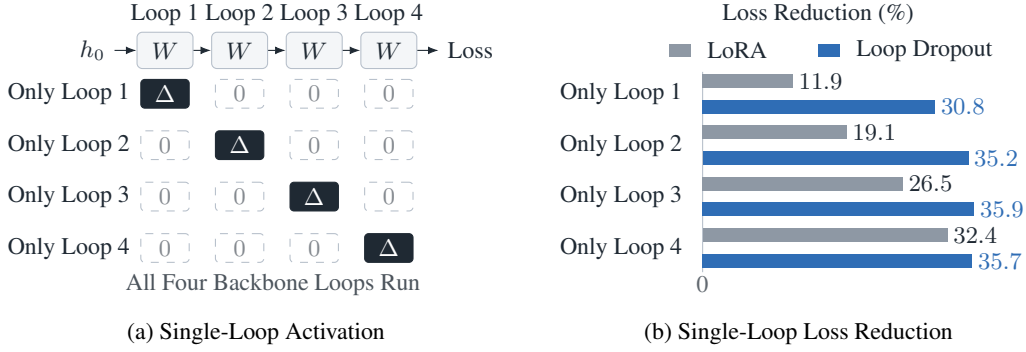}
\endgroup
\caption{\textbf{Loop Dropout reduces the late-loop bias of shared adaptation.}
Ouro~\citep{zhu2025ouro} uses four loops through a shared transformer block by default.
To probe adaptation across loops after GSM8K fine-tuning, we activate the adapter in one loop at a time while retaining all four backbone loops.
LoRA's loss reduction falls from 32\% at the fourth loop to 12\% at the first; Loop Dropout maintains 31--36\% across positions.
All reductions are relative to the frozen model and use final-loop teacher-forced test loss.}
\label{fig:teaser}
\end{figure}

Dropout encourages features to remain useful across different combinations of other features~\citep{srivastava2014dropout}.
For a shared adapter, the complete application at one loop is a natural unit of regularization: masking it trains the update across different recurrent states and combinations of other applications.
However, we find that randomly omitting adapter applications alone does not improve task performance.
Stochastic omission creates a training--inference mismatch: it reduces the expected training update, while inference applies the full update at every loop.

We therefore introduce \method, which couples stochastic masking of adapter applications with inverse-survival rescaling to promote effective adaptation across loops.
Each retained update is divided by its survival probability, preserving the expected update strength during training.
This design trains the shared update under varying application patterns while keeping every backbone loop active.
At inference, all applications are active and the adapter follows standard LoRA, with no additional trainable parameters or inference computation.

\paperref{Figure}{fig:teaser_reduction} shows that \method maintains loss reductions of 31--36\% across all four loops, narrowing the disparity between early and late applications.
When adapters trained with four loops are evaluated at eight loops without further fine-tuning, \method outperforms LoRA by 4.75 percentage points on GSM8K.

Extensive experiments demonstrate that \method delivers consistent gains on mathematical reasoning and improves general instruction tuning.
\method outperforms existing LoRA variants and adapter regularizers such as CoTo~\citep{zhuang2025coto}, highlighting the value of tailoring adaptation to the recurrent structure of looped language models.

Our contributions are as follows:
\begin{itemize}
\item We identify a pronounced late-loop bias in standard LoRA fine-tuning, exposing the challenge of learning shared updates that remain effective across loops.
\item We introduce \method, coupling independent masks over complete applications of the shared update with inverse-survival rescaling to promote effective adaptation across loops while retaining standard LoRA inference.
\item We demonstrate improved mathematical adaptation and transfer across model sizes, adapter ranks and training recipes, supported by stronger early-loop adaptation, zero-shot generalization to deeper recurrence and comparisons with alternative regularizers.
\end{itemize}

\section{Preliminaries}
\label{sec:preliminaries}

\paragraph{Looped computation.}
A looped language model embeds an input sequence into $h_0$ and applies a transformer block $F$ with shared parameters $W$ for $K$ loops:
\begin{equation}
h_t=F(h_{t-1};W),\qquad t=1,\ldots,K,\qquad z=\mathrm{Head}(h_K).
\label{eq:loop}
\end{equation}
We call each application of the shared transformer block a \emph{loop}; the recurrence depth $K$ is the number of loops.
Increasing $K$ adds computation without another copy of the block parameters.
We use the base Ouro models~\citep{zhu2025ouro}, whose default recurrence depth is $K=4$, and fix the depth for all examples rather than using their adaptive exit gates.

\paragraph{Shared low-rank updates.}
LoRA~\citep{hu2022lora} adapts a frozen matrix $W_i\in\R^{d_{\mathrm{out}}\times d_{\mathrm{in}}}$ through $\Delta_i=\frac{\alpha}{r}B_iA_i$, where $A_i\in\R^{r\times d_{\mathrm{in}}}$, $B_i\in\R^{d_{\mathrm{out}}\times r}$, and $r$ is the rank.
We write $\Delta$ for the collection of these updates across the recurrent block.
Standard fine-tuning then optimizes
\begin{equation}
\mathcal{L}_{\mathrm{LoRA}}(\Delta)=\E_{(x,y)}\,\CE\big(\mathrm{Head}(h_K),y\big),
\qquad h_t=F(h_{t-1};W+\Delta),
\label{eq:shared_update}
\end{equation}
with loss on the target tokens $y$.
Each $\Delta_i$ has one pair of trainable factors shared across all $K$ loops, and its gradient includes contributions from every application.
Training separate factors at each loop is an alternative, but multiplies the adapter parameter count by $K$ at a fixed rank.
We compare both equal-rank and equal-parameter versions of this alternative.

\section{Loop Dropout}
\label{sec:method}

The late-loop bias identified in \paperref{Section}{sec:introduction} motivates learning a shared update that remains effective throughout the recurrent computation.
\method couples loop-level masking with inverse-survival rescaling: masking trains the update under varying combinations of its applications, while rescaling preserves its expected strength at every loop.
\paperref{Figure}{fig:overview} illustrates the gates used in training and at inference.

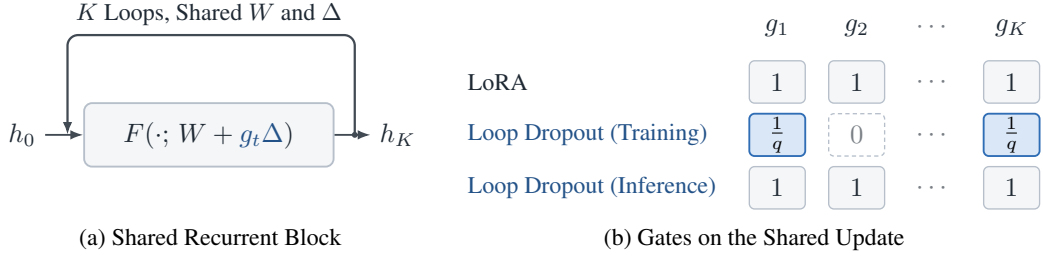
\begin{figure}[t]
\centering
\begingroup
\tikzset{
  block/.style={draw=FigOutline,fill=FigNeutral,line width=0.7pt,rounded corners=3pt,
    minimum width=3.30cm,minimum height=0.84cm,inner sep=4pt},
  gate/.style={draw=FigOutline,fill=FigNeutral,line width=0.5pt,rounded corners=2pt,
    minimum width=0.74cm,minimum height=0.56cm,inner sep=0pt,text=FigInk},
  kept/.style={gate,keptcell,line width=0.7pt},
  dropped/.style={gate,fill=white,draw=FigOutline,dash pattern=on 1.6pt off 1.2pt,text=FigInk!60},
  flow/.style={-{Latex[length=1.7mm]},draw=FigInk!85,line width=0.7pt},
  recurrence/.style={flow,line width=0.85pt,rounded corners=4pt},
  rowlabel/.style={font=\small,anchor=west,inner sep=0pt,text=FigInk},
  every node/.style={text=FigInk},
}
\begin{subfigure}[b]{0.42\textwidth}
\centering
\begin{tikzpicture}[x=1cm,y=1cm,font=\normalsize]
\path[use as bounding box] (0,-0.10) rectangle (5.55,2.72);
\node (hin) at (0.28,0.90) {$h_0$};
\node[block] (shared) at (2.75,0.90) {$F(\cdot;\,W+{\color{FigMainText}g_t\Delta})$};
\node (hout) at (5.25,0.90) {$h_K$};
\draw[flow] (hin.east) -- (shared.west);
\draw[flow] (shared.east) -- (hout.west);
\draw[recurrence] (4.68,0.90) -- (4.68,2.22) -- (0.88,2.22) -- (0.88,0.90);
\fill[FigInk!85] (4.68,0.90) circle (1.1pt);
\node[font=\small] at (2.75,2.50) {$K$ Loops, Shared $W$ and $\Delta$};
\end{tikzpicture}
\caption{Shared Recurrent Block}
\label{fig:overview_block}
\end{subfigure}%
\hfill
\begin{subfigure}[b]{0.55\textwidth}
\centering
\begin{tikzpicture}[x=1cm,y=1cm,font=\normalsize]
\path[use as bounding box] (6.25,-0.10) rectangle (13.85,2.72);
\node[rowlabel] at (6.25,1.60) {LoRA};
\node[rowlabel,text=FigMainText] at (6.25,0.90) {Loop Dropout (Training)};
\node[rowlabel,text=FigMainText] at (6.25,0.20) {Loop Dropout (Inference)};
\foreach \t/\xx in {1/10.35,2/11.40,K/13.45} {
  \node at (\xx,2.32) {$g_{\t}$};
  \node[gate] at (\xx,1.60) {$1$};
  \node[gate] at (\xx,0.20) {$1$};
}
\foreach \yy in {2.32,1.60,0.90,0.20} {\node[text=FigInk!60] at (12.42,\yy) {$\cdots$};}
\foreach \xx in {10.35,13.45} {\node[kept] at (\xx,0.90) {$\tfrac{1}{q}$};}
\node[dropped] at (11.40,0.90) {$0$};
\end{tikzpicture}
\caption{Gates on the Shared Update}
\label{fig:overview_gates}
\end{subfigure}
\endgroup
\caption{\textbf{Loop Dropout regularizes shared adaptation across loops.}
The shared block $F$ reuses the same $W$ and $\Delta$ over $K$ loops.
Following \papereqref{Eq.}{eq:loop_dropout}, training uses $g_t=b_t/q$, where $q=1-p$ and $b_t\sim\Bernoulli(q)$ independently for each example and loop.
The training row illustrates one sampled pattern: retained updates have gate $1/q$, and dropped updates have gate $0$.
LoRA and inference use $g_t=1$ at every loop.
A zero gate removes only $\Delta$; every backbone loop remains active.}
\label{fig:overview}
\end{figure}

\subsection{Loop-Level Masking}
\label{sec:loop_dropout}

For each training example, \method samples independent masks $b_t\sim\Bernoulli(q)$, where $q=1-p\in(0,1]$ is the survival probability, and computes
\begin{equation}
h_t=F(h_{t-1};W+g_t\Delta),\qquad g_t=\frac{b_t}{q},\qquad t=1,\ldots,K.
\label{eq:loop_dropout}
\end{equation}
One mask is shared by all adapted modules and token positions within a loop, as illustrated in \paperref{Figure}{fig:overview_gates}.
If $b_t=0$, the loop uses only the pretrained weights; otherwise, it uses the shared update scaled by $1/q$.
The loss is evaluated at the final loop as in \papereqref{Eq.}{eq:shared_update}, all backbone weights remain frozen, and \paperref{Algorithm}{alg:loop_dropout} summarizes the procedure.

Masking a complete application changes both the input states reaching later loops and the set of adapter applications that contribute to the final prediction.
Each active application is thus trained with varying combinations of earlier and later applications.
Sharing one mask across all adapted modules makes the complete adapter application the unit of regularization.

\begin{algorithm}[t]
\caption{Training with Loop Dropout}
\label{alg:loop_dropout}
\begin{algorithmic}[1]
\Require frozen block $F(\cdot;W)$, shared LoRA factors $\Delta$, depth $K$, survival probability $q$
\For{each training example $(x,y)$ in a mini-batch}
  \State $h_0\gets\mathrm{Embed}(x)$
  \For{$t=1,\ldots,K$}
    \State $b_t\sim\Bernoulli(q)$ independently; $g_t\gets b_t/q$
    \State $h_t\gets F(h_{t-1};W+g_t\Delta)$ \Comment{same gate for all modules and tokens}
  \EndFor
  \State $\ell_{x,y}\gets\CE(\mathrm{Head}(h_K),y)$
\EndFor
\State Update the LoRA factors using the mean mini-batch loss.
\State \textbf{Inference:} set $g_t=1$ at every loop.
\end{algorithmic}
\end{algorithm}

\subsection{Inverse-Survival Rescaling}
\label{sec:rescale}

Masking alone reduces each loop's expected training update from $\Delta$ to $q\Delta$, while inference uses the full update $\Delta$.
The factor $1/q$ in \papereqref{Eq.}{eq:loop_dropout} compensates for this reduction, matching each loop's expected training update to its inference update.
Writing $g_t=1+\varepsilon_t$ with $\varepsilon_t=(b_t-q)/q$, the effective weights of loop $t$ decompose as
\begin{equation}
W+g_t\Delta
=\underbrace{W+\Delta}_{\text{inference weights}}
+\underbrace{\varepsilon_t\Delta}_{\text{zero-mean perturbation}},
\qquad
\E[\varepsilon_t]=0,
\qquad
\operatorname{Cov}(\varepsilon)=\frac{p}{q}\,I_K.
\label{eq:gate_decomposition}
\end{equation}
The perturbation acts along the learned update and leaves the pretrained weights unchanged.
At $p=1/2$, each loop applies either $2\Delta$ or no update with equal probability, while its mean update remains $\Delta$.
This is the mean-preserving convention of inverted dropout~\citep{srivastava2014dropout}, applied to complete adapter applications.

This property extends to the composed recurrent output at first order.
Let $z_\eta(g)$ denote the final logits when loop $t$ uses $W+\eta g_t\Delta$, where $\eta$ is an auxiliary update scale for expansion about the frozen computation at $\eta=0$.

\begin{lemma}[First-order consistency of recurrent adaptation]
\label{lem:first_order_rescale}
Fix an input, update $\Delta$, depth $K$ and survival probability $q\in(0,1]$.
Suppose $F$ is twice continuously differentiable in its state and weights, and $\mathrm{Head}$ is twice continuously differentiable, near the frozen trajectory.
For independent $b_t\sim\Bernoulli(q)$, as $\eta\to0$,
\begin{equation}
\begin{aligned}
\E_b\,z_\eta(b/q)&=z_\eta(\ones)+O(\eta^2),\\
\E_b\,z_\eta(b)&=z_{q\eta}(\ones)+O(\eta^2).
\end{aligned}
\label{eq:first_order_rescale}
\end{equation}
\end{lemma}

Each first-order contribution includes an update's effect propagated through all subsequent backbone loops.
Rescaling preserves the mean of their sum, while unscaled masking multiplies it by $q$.
\paperref{Appendix}{app:first_order_rescale} derives these contributions and bounds the second-order remainder.
Thus rescaling aligns the leading adaptation effect during training with that at inference, complementing the varying application patterns produced by masking.
The comparison in \paperref{Table}{tab:mask_controls} shows higher accuracy when loop-level masking is coupled with this rescaling.

\subsection{Training Across Application Patterns}
\label{sec:mixture}

The two components jointly train the shared update across a distribution of application patterns while preserving its expected strength.
Let $\ell_\Delta(g)$ denote the loss for one example under a gate vector $g$, and let $\ones_S$ indicate the subset $S$ of loops with active updates.
The objective is
\begin{equation}
\E_b\,\ell_\Delta(b/q)
=\sum_{S\subseteq\{1,\ldots,K\}}q^{|S|}p^{K-|S|}\,
\ell_\Delta(\ones_S/q).
\label{eq:mixture}
\end{equation}
This weighted average jointly optimizes single-loop and multi-loop applications of the same update, exposing the shared parameters to both sparse and dense application patterns.
\paperref{Section}{sec:early_adaptation} examines how the learned update's effectiveness changes across loop positions.

For a fixed learned update, centering the gates at $\ones$ also gives a local regularization view of this objective.
Expanding \papereqref{Eq.}{eq:mixture} around the inference gate yields
\begin{equation}
\E_b\,\ell_\Delta(b/q)
=\ell_\Delta(\ones)
+\frac{p}{2q}\sum_{t=1}^{K}v_t^{\top}H\,v_t
+R,
\qquad
v_t=\frac{\partial z}{\partial g_t}\bigg|_{g=\ones},
\label{eq:penalty}
\end{equation}
where $H\succeq0$ is the Hessian of the cross-entropy with respect to the logits $z$ and $R$ collects the higher-order terms together with the part of the gate Hessian that involves second derivatives of $z$.
The first term is the standard LoRA objective; the nonnegative curvature term describes local regularization along each application's logit sensitivity, with strength $p/(2q)$ equal to half the gate variance~\citep{wager2013dropout,bishop1995noise}.
\paperref{Appendix}{app:rescale_analysis} gives the derivation, the exact scaling relations and the gate covariance of the controls in \paperref{Section}{sec:mask_controls}.

Our default uses $p=0.5$ and requires only a scalar gate on each adapter output.
The method adds no trainable parameters or inference computation to standard LoRA.
Implementation details and alternative mask distributions are given in \paperref{Appendices}{app:implementation} and~\papernumref{app:variants}.

\section{Experiments}
\label{sec:experiments}

We evaluate \method for mathematical reasoning on Ouro-1.4B, Ouro-2.6B, LoopUS-Qwen3-4B~\citep{park2026loopus}, and LoopCoder-V2~\citep{yang2026loopcoderv2}.
We further study instruction tuning on Ouro-1.4B and analyze early-loop adaptation, the roles of masking and rescaling, and robustness across adapter ranks and recurrence depths.

\subsection{Experimental Setup}
\label{sec:experimental_setup}

\paragraph{Models and tasks.}
We use the base Ouro models with four loops.
Mathematical fine-tuning follows the MetaMathQA pipeline of LoRA-Pro~\citep{yu2024metamath,wang2025lorapro}: one epoch on 100k GSM-type examples, followed by zero-shot evaluation on GSM8K~\citep{cobbe2021gsm8k} and MATH-500~\citep{hendrycks2021math,lightman2024verify} with the MetaMath scorer.
Instruction tuning of Ouro-1.4B uses a 100k-example draw from T\"ulu~2~\citep{ivison2023tulu2} and evaluates code generation with EvalPlus~\citep{chen2021codex,austin2021program,liu2023evalplus}, knowledge and reasoning with MMLU and BBH~\citep{hendrycks2021mmlu,suzgun2023bbh}, truthfulness with TruthfulQA~\citep{lin2021truthfulqa}, and instruction following with IFEval~\citep{zhou2023ifeval}.
A smaller recipe fine-tunes directly on GSM8K and supports the initial activation diagnostic and extended robustness studies.
\paperref{Appendix}{app:experimental_details} gives the training recipes and scoring rules.

\paragraph{Methods.}
Default adapters use rank 16 and $\alpha/r=2$ on the seven projections of the recurrent block, giving 15.1M trainable parameters on Ouro-1.4B and 30.3M on Ouro-2.6B.
\method uses $p=0.5$.
We compare with LoRA~\citep{hu2022lora}, LoRA+~\citep{hayou2024loraplus}, CoTo~\citep{zhuang2025coto} and LoRA Dropout~\citep{lin2024loradropout}, and isolate masking and sharing through the controls in \paperref{Section}{sec:analysis}.
\paperref{Appendix}{app:lr_shifted} also compares rsLoRA~\citep{kalajdzievski2023rslora} under the direct GSM8K recipe (\paperref{Table}{tab:gsm8k_scale}).
We tune the learning rate over $\{1.25,2.5,5,10,20\}\times10^{-5}$ separately for each method in the main mathematical comparisons and the rank and sharing studies, with the shifted LoRA+ grid specified below.
The training variants in \paperref{Table}{tab:mask_controls}, including \method, use the same optimizer settings, including the learning rate, to isolate the training rule.

\paragraph{Training and reporting.}
Comparisons match training examples, optimizer steps, sequence length and evaluation protocol.
Main mathematical results report mean and sample SD over three training seeds.
LoRA+ uses $\eta_B/\eta_A=4$ and the shifted candidate grid in \paperref{Appendix}{app:hyperparameters}.

\subsection{Overall Performance}
\label{sec:main_results}

\begin{table}[t]
\centering
\footnotesize
\setlength{\tabcolsep}{3.5pt}
\caption{\textbf{Mathematical reasoning performance.} Accuracy in \%, mean $\pm$ SD over three training seeds per method. Bold marks the highest mean within each model and benchmark.}
\label{tab:math_main}
\begin{tabular*}{\textwidth}{@{\extracolsep{\fill}}lcccccc}
\toprule
& \multicolumn{3}{c}{Ouro-1.4B} & \multicolumn{3}{c}{Ouro-2.6B} \\
\cmidrule(lr){2-4}\cmidrule(lr){5-7}
Benchmark & LoRA & LoRA+ & \method & LoRA & LoRA+ & \method \\
\midrule
GSM8K & 85.34 $\pm$ 0.64 & 85.54 $\pm$ 0.23 & \textbf{86.53 $\pm$ 0.44} & 87.62 $\pm$ 0.52 & 87.21 $\pm$ 0.38 & \textbf{88.73 $\pm$ 0.31} \\
MATH-500 & 38.87 $\pm$ 1.67 & 38.13 $\pm$ 0.42 & \textbf{47.53 $\pm$ 2.61} & 42.87 $\pm$ 1.17 & 43.40 $\pm$ 0.53 & \textbf{47.20 $\pm$ 0.53} \\
\bottomrule
\end{tabular*}

\par\medskip
\centering
\footnotesize
\setlength{\tabcolsep}{6pt}
\caption{\textbf{Instruction-tuning performance.} T\"ulu~2 recipe; accuracy in \%, mean $\pm$ SD over three training seeds. HumanEval+ uses EvalPlus; IFEval reports prompt-level strict accuracy. Average uses the six-benchmark set in \paperref{Appendix}{app:tulu_returned}, computed per seed. Bold marks the higher mean in each column.}
\label{tab:tulu_main}
\begin{tabular*}{\textwidth}{@{\extracolsep{\fill}}lccccc@{}}
\toprule
Method & HumanEval+ & MMLU & TruthfulQA MC2 & IFEval & Average (6) \\
\midrule
LoRA & 67.68 $\pm$ 3.05 & 68.64 $\pm$ 0.40 & 47.32 $\pm$ 1.02 & 46.33 $\pm$ 1.26 & 60.98 $\pm$ 0.66 \\
\method & \textbf{69.92 $\pm$ 0.35} & \textbf{68.91 $\pm$ 0.13} & \textbf{48.19 $\pm$ 0.68} & \textbf{46.46 $\pm$ 1.02} & \textbf{61.51 $\pm$ 0.09} \\
\bottomrule
\end{tabular*}
\end{table}

\paragraph{Mathematical reasoning and transfer.}
\paperref{Table}{tab:math_main} shows that \method improves both mathematical benchmarks at both model sizes.
On Ouro-1.4B, GSM8K accuracy increases by 1.19 percentage points over LoRA, while MATH-500 rises from 38.87\% to 47.53\%, an 8.67-point gain.
On Ouro-2.6B, the corresponding gains are 1.11 and 4.33 points.
Both benchmarks improve in every training seed at each size.
The larger gains on MATH-500 show stronger transfer from grade-school training problems to competition mathematics.
This pattern suggests that regularizing the shared update helps transfer learned reasoning patterns, without external knowledge or additional supervision.
\method also improves GSM8K and MATH-500 accuracy over LoRA on LoopUS-Qwen3-4B and LoopCoder-V2.
Due to space constraints, we report the full results in \paperref{Appendix}{app:family_comparison}.

\paragraph{Instruction tuning.}
On Ouro-1.4B, \method improves HumanEval+ by 2.24 points and TruthfulQA MC2 by 0.87 points over LoRA.
\paperref{Table}{tab:tulu_main} also shows higher means on MMLU and IFEval.
The six-benchmark average is 61.51, compared with 60.98 for LoRA.
\paperref{Appendix}{app:tulu_returned} reports all nine instruction-tuning metrics, whose average is also higher for \method than for LoRA.

\subsection{Comparison with Existing Adaptation Methods}
\label{sec:expanded_baselines}

We compare \method with alternative low-rank adaptation and regularization methods under the same mathematical training recipe and hyperparameter tuning budget.
\paperref{Table}{tab:matched_expansion} places accuracy alongside the measured cost of training.

\begin{table}[t]
\centering
\small
\setlength{\tabcolsep}{5pt}
\caption{\textbf{Comparison with existing adaptation methods.} Ouro-1.4B, rank 16, mathematical recipe. Accuracy is mean $\pm$ SD over three training seeds; training time and peak memory are their means. Displayed costs use an H100 80GB and exclude learning-rate search and evaluation. All methods use 15.1M adapter parameters.}
\label{tab:matched_expansion}
\begin{tabular}{lcccc}
\toprule
Method & GSM8K & MATH-500 & Train (min) & Memory (GB) \\
\midrule
LoRA & 85.34 $\pm$ 0.64 & 38.87 $\pm$ 1.67 & 110.6 & 45.36 \\
LoRA+ & 85.54 $\pm$ 0.23 & 38.13 $\pm$ 0.42 & 112.8 & 45.40 \\
CoTo-on-Ouro & 85.60 $\pm$ 0.76 & 37.80 $\pm$ 1.06 & 94.1 & 45.28 \\
LoRA Dropout & 85.77 $\pm$ 0.70 & 42.13 $\pm$ 0.31 & 444.9 & 45.37 \\
\method & \textbf{86.53 $\pm$ 0.44} & \textbf{47.53 $\pm$ 2.61} & 128.9 & 45.36 \\
\bottomrule
\end{tabular}
\end{table}

\paragraph{Comparison with adapter regularizers.}
CoTo-on-Ouro shares each physical layer's adapter switch across its recurrent uses and increases the active fraction during training; LoRA Dropout introduces sparsity within the low-rank update.
\method couples masks on complete loop applications with inverse-survival rescaling, tailoring regularization to the recurrent use of the shared update.
On MATH-500, it exceeds CoTo-on-Ouro by 9.73 points and LoRA Dropout by 5.40 points, with positive differences in all three training seeds against both methods.
Its GSM8K margins are 0.94 and 0.76 points, respectively.
These comparisons show that regularizing the repeated applications of the update yields stronger mathematical transfer than the evaluated alternatives.
\paperref{Appendix}{app:coto_seeds} gives the baseline configurations and per-seed results.
We test \method with DoRA and VeRA adapters in \paperref{Appendix}{app:adapter_compatibility} and report a full fine-tuning reference in \paperref{Appendix}{app:full_finetuning}.

\paragraph{Efficiency comparison.}
\label{sec:task_scope}
\method improves mathematical accuracy at the parameter count and inference computation of standard LoRA.
Relative to LoRA Dropout, it reduces measured training time by 71\% while improving accuracy on both benchmarks.
\paperref{Appendix}{app:compute} specifies the hardware and timing procedure.

\subsection{Early-Loop Adaptation}
\label{sec:early_adaptation}

We now examine whether the task gains are accompanied by stronger early-loop adaptation, addressing the late-loop bias identified in \paperref{Section}{sec:introduction}.
We first activate a trained adapter at only one of four loops, retain all four backbone loops, and measure the reduction in final-loop teacher-forced loss relative to the same frozen model.
This intervention isolates how much the learned update reduces prediction loss when applied alone at each position, without shortening the recurrent computation.

\paragraph{Reducing the late-loop bias.}
\paperref{Figure}{fig:teaser_reduction} quantifies the imbalance after standard LoRA fine-tuning on GSM8K: \method narrows the gap in loss reduction between the first and final applications from 20.5 to 4.9 percentage points.
The improvement is largest at the earliest loop, strengthening the applications that standard fine-tuning leaves least effective.

\paragraph{Generation from individual applications.}
After MetaMath-GSM fine-tuning, we evaluate GSM8K generation with the adapter active at one loop while retaining all four backbone loops.
\paperref{Table}{tab:single_loop_generation} compares all four activation positions.
At the second loop, \method reaches 64.72\% accuracy, compared with 0.30\% for LoRA; at the third and fourth loops, it reaches 86.23\% and 86.18\%, compared with 2.63\% and 6.07\%.
These results show that the update trained with Loop Dropout can support generation with a single application at loops two through four.

\begin{table}[t]
\centering
\small
\setlength{\tabcolsep}{5pt}
\caption{\textbf{GSM8K generation with one adapter application.} Ouro-1.4B after MetaMath-GSM fine-tuning; accuracy in \%, mean $\pm$ SD over three seeds. Only the indicated adapter application is enabled, with all four backbone loops retained.}
\label{tab:single_loop_generation}
\begin{tabular}{lcccc}
\toprule
Method & Only loop 1 & Only loop 2 & Only loop 3 & Only loop 4 \\
\midrule
LoRA & 0.03 $\pm$ 0.04 & 0.30 $\pm$ 0.46 & 2.63 $\pm$ 3.07 & 6.07 $\pm$ 4.42 \\
\method & 1.31 $\pm$ 0.69 & 64.72 $\pm$ 13.56 & 86.23 $\pm$ 0.84 & 86.18 $\pm$ 0.78 \\
\bottomrule
\end{tabular}
\end{table}

With every adapter application enabled, the readout diagnostic in \paperref{Appendix}{app:metamath_diagnostics} also shows lower teacher-forced loss at every MATH-500 readout, most strongly at the first loop.
This connects stronger early-loop prediction to the normal adapted computation.

\FloatBarrier

\subsection{Ablation Studies}
\label{sec:analysis}

The ablations examine the contributions of loop-level masking and inverse-survival rescaling, and how their benefit depends on sharing the update.

\paragraph{Effects of masking and rescaling.}
\label{sec:mask_controls}
We examine which aspects of loop-level masking and inverse-survival rescaling account for the gains of \method.
\paperref{Table}{tab:mask_controls} compares \method with masking and perturbation baselines, using the same shared adapter and inference procedure.
All methods retain all four backbone loops and apply the full adapter update at inference.

\begin{description}
\item[LoRA.] The shared adapter is active at every loop with unit gate, providing the unmasked reference.

\item[Unscaled.] We keep the loop-level Bernoulli masks but remove inverse-survival rescaling, using $g_t=b_t$ instead of $b_t/q$.
This tests the role of preserving expected update strength.

\item[Averaged gates.] We replace the sampled Loop Dropout gates by their mean across loops: for example, $[0,2,0,2]$ becomes $[1,1,1,1]$.
This preserves their sum while removing loop-to-loop variation, testing the role of different application patterns beyond a varying common strength.

\item[Module-wise.] We sample a separate rescaled mask for each adapted linear map at each loop.
This tests masking individual maps against removing a complete adapter application.

\item[Low-rank weight noise.] We add a random low-rank matrix perturbation to each learned adapter update.
This tests random weight perturbations with the same expected energy as \method.

\item[Parallel noise.] We add scalar noise along the learned update, sharing the scalar across adapted maps and tokens within a loop.
The gates match the mean and covariance of the Loop Dropout gates but permit negative values and have no probability mass at zero, testing whether matching the first two moments reproduces the benefit of discrete masking.

\item[Loop Dropout.] One rescaled gate $b_t/q$ is shared across all adapted maps and tokens within each loop, removing complete adapter applications while preserving expected update strength.
\end{description}

Both noise baselines use unit strength $c=1$ for the stated matching conditions; \paperref{Appendix}{app:implementation} gives their precise distributions.

\begin{table}[htbp]
\centering
\small
\caption{\textbf{Ablating loop-level masking and inverse-survival rescaling.} The baselines separate update scaling, gate variation across loops, masking granularity and matched noise. All rows use shared rank-16 adapters on Ouro-1.4B, retain four backbone loops and apply the full adapter update at inference. MetaMath-GSM recipe; accuracy in \%, mean $\pm$ SD over three seeds. Noise strength $c=1$ matches the expected perturbation energy of \method. Per-seed configurations and results are in \paperref{Tables}{tab:control_seeds} and~\papernumref{tab:returned_seeds}.}
\label{tab:mask_controls}
\begin{tabular}{lcc}
\toprule
Method & GSM8K & MATH-500 \\
\midrule
LoRA & 85.32 $\pm$ 0.92 & 40.13 $\pm$ 1.62 \\
Unscaled & 84.71 $\pm$ 0.19 & 39.33 $\pm$ 1.70 \\
Averaged gates & 85.57 $\pm$ 0.88 & 38.60 $\pm$ 1.71 \\
Module-wise & 85.52 $\pm$ 0.62 & 41.13 $\pm$ 0.76 \\
Low-rank weight noise & 85.70 $\pm$ 0.54 & 39.13 $\pm$ 0.83 \\
Parallel noise & 85.62 $\pm$ 0.64 & 37.93 $\pm$ 1.33 \\
\method & \textbf{87.57 $\pm$ 0.55} & \textbf{44.00 $\pm$ 1.40} \\
\bottomrule
\end{tabular}
\end{table}

\method achieves the highest accuracy on both benchmarks in \paperref{Table}{tab:mask_controls}.
Coupling the loop masks with rescaling improves over unscaled masking, and masking complete applications gives higher MATH-500 accuracy than independent module-wise masks.
The gain over Averaged gates supports varying which loops apply the update beyond varying their common strength.
\method also outperforms both matched noise baselines; the additional configurations in \paperref{Table}{tab:returned_seeds} retain this ordering.
Together, these comparisons support training across combinations of complete applications while preserving expected update strength.

\paragraph{Shared and independent adapters.}
\label{sec:sharing_mask}
We examine how sharing adapter parameters across loops affects the benefit of \method in \paperref{Table}{tab:sharing_mask}.
For each adapted linear map, \textbf{Shared} reuses the same LoRA factors $(A,B)$ at all four loops, whereas \textbf{Independent} learns a separate pair $(A_t,B_t)$ for each loop $t$.
Both settings keep the backbone weights shared and frozen.
Independent rank four matches the shared rank-16 adapter's total parameter count (15.1M), while independent rank sixteen matches its per-loop capacity (60.6M parameters in total).
Each row compares training with and without \method under the same hyperparameter candidate budget, with all adapters active at inference.

\begin{table}[htbp]
\centering
\footnotesize
\setlength{\tabcolsep}{3pt}
\caption{\textbf{Loop Dropout with shared and independent adapters.} Rank is per loop, and Params counts all trainable adapter parameters. Both parameterizations use complete-application masking and inverse-survival rescaling with \method; all adapters are active at inference. Ouro-1.4B, MetaMath-GSM recipe with development selection over five learning rates per method; accuracy in \%, mean $\pm$ SD over three seeds.}
\label{tab:sharing_mask}
\begin{tabular}{lrc cccc}
\toprule
& & & \multicolumn{2}{c}{Without Loop Dropout} & \multicolumn{2}{c}{With Loop Dropout} \\
\cmidrule(lr){4-5}\cmidrule(lr){6-7}
Adapters & Rank & Params & GSM8K & MATH-500 & GSM8K & MATH-500 \\
\midrule
Shared & 16 & 15.1M & 85.34 $\pm$ 0.64 & 38.87 $\pm$ 1.67 & 86.53 $\pm$ 0.44 & 47.53 $\pm$ 2.61 \\
Independent & 4 & 15.1M & 84.99 $\pm$ 0.84 & 39.40 $\pm$ 2.11 & 87.74 $\pm$ 0.16 & 45.60 $\pm$ 1.20 \\
Independent & 16 & 60.6M & 84.94 $\pm$ 0.42 & 37.33 $\pm$ 1.79 & 86.58 $\pm$ 0.20 & 43.87 $\pm$ 0.90 \\
\bottomrule
\end{tabular}
\end{table}

\method improves both benchmarks in every seed for all three adapter settings.
The MATH-500 gain is 8.67 points with shared adapters, compared with 6.20 and 6.53 points with independent rank-four and rank-sixteen adapters.
The shared update combines the highest MATH-500 accuracy with 15.1M trainable parameters and can be merged once into a single weight set reused at every loop; independent updates produce different effective weights at different loops.
\paperref{Appendix}{app:sharing_settings} specifies the training and evaluation settings.

The larger-model comparison in \paperref{Table}{tab:ouro26_tuned} also distinguishes sharing from adapter capacity.
On Ouro-2.6B, \method exceeds independently tuned per-loop rank-four and rank-sixteen adapters by 3.60 and 5.80 points on MATH-500, respectively.
Allocating separate updates to the loops therefore does not recover the transfer performance of the shared update trained with Loop Dropout.

\paragraph{Additional training variants.}
\paperref{Table}{tab:variants} compares fixed-size masks, positional profiles, schedules and learned gates under the direct GSM8K recipe.
No variant exceeds uniform Bernoulli masking at both learning rates, so we use uniform masks without an additional schedule or learned gates.
\paperref{Appendix}{app:extended_ablations} reports these variants in full.

\subsection{Robustness Across Adapter Ranks and Recurrence Depths}
\label{sec:model_generalization}

We test whether the benefits of regularizing shared adaptation extend across adapter capacities and beyond the training recurrence depth.
The rank study uses the MetaMath-GSM recipe, and the depth study uses the direct GSM8K recipe.

\paragraph{Gains across ranks.}
With the same tuning budget for each method and rank, \method improves mean accuracy on both mathematical benchmarks at all four ranks in \paperref{Figure}{fig:rank_overview}; \paperref{Table}{tab:rank_tuned} gives the numerical results.
MATH-500 gains are 7.07, 8.67, 2.87 and 2.33 points at ranks 4, 16, 64 and 128, respectively.
At rank 16, the gain is positive in all three seeds and across all seven subject means.
The decomposition in \paperref{Appendix}{app:rank_tuned} attributes 7.87 points of this 8.67-point gain to problems where both methods produce extractable answers.
The transfer benefit therefore persists across adapter capacities after tuning each method independently.

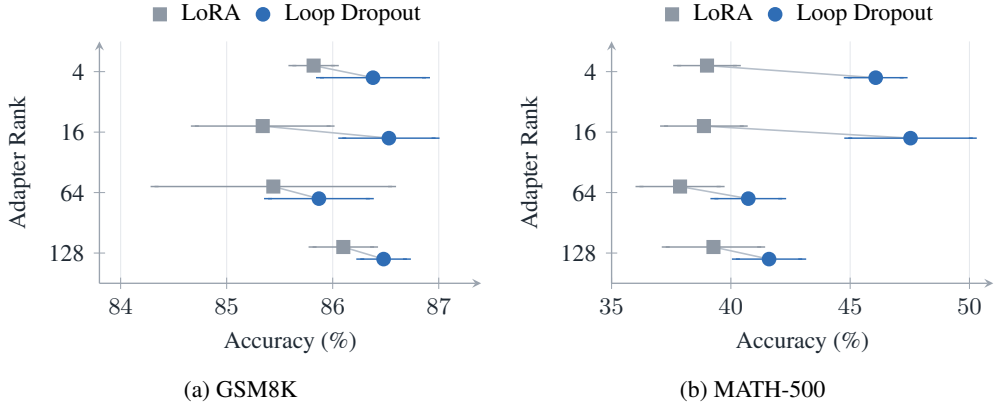
\begin{figure}[!htbp]
\centering
\begin{subfigure}[b]{0.465\textwidth}
\centering
\captionsetup{justification=centering}
\begin{tikzpicture}
\begin{axis}[scale only axis,width=5.05cm,height=3.2cm,
 xmin=83.8,xmax=87.4,ymin=0.5,ymax=4.5,
 xtick={84,85,86,87},ytick={1,2,3,4},yticklabels={128,64,16,4},
 xlabel={Accuracy (\%)},ylabel={Adapter Rank},
 legend columns=2,
 legend style={at={(0.5,1.035)},anchor=south,inner sep=1pt,column sep=4pt},
 ymajorgrids=false,xmajorgrids=true,
 error bars/y dir=none,error bars/x dir=both,error bars/x explicit,
 error bars/error mark options={rotate=0,mark size=1.5pt,line width=0.55pt}]
\foreach \yb/\xl/\xr in {4/85.82/86.38,3/85.34/86.53,2/85.44/85.87,1/86.10/86.48}{
\edef\next{\noexpand\draw[FigOutline,line width=0.6pt] (axis cs:\xl,\yb+0.10)--(axis cs:\xr,\yb-0.10);}\next}
\addplot+[basecurve,only marks] coordinates {(85.82,4.10)+-(0.20,0) (85.34,3.10)+-(0.64,0) (85.44,2.10)+-(1.12,0) (86.10,1.10)+-(0.29,0)};
\addlegendentry{LoRA}
\addplot+[ourscurve,only marks] coordinates {(86.38,3.90)+-(0.50,0) (86.53,2.90)+-(0.44,0) (85.87,1.90)+-(0.48,0) (86.48,0.90)+-(0.22,0)};
\addlegendentry{Loop Dropout}
\end{axis}
\end{tikzpicture}
\caption{GSM8K}
\label{fig:rank_gsm8k}
\end{subfigure}\hspace{0.02\textwidth}%
\begin{subfigure}[b]{0.465\textwidth}
\centering
\captionsetup{justification=centering}
\begin{tikzpicture}
\begin{axis}[scale only axis,width=5.05cm,height=3.2cm,
 xmin=35,xmax=51,ymin=0.5,ymax=4.5,
 xtick={35,40,45,50},ytick={1,2,3,4},yticklabels={128,64,16,4},
 xlabel={Accuracy (\%)},ylabel={Adapter Rank},
 legend columns=2,
 legend style={at={(0.5,1.035)},anchor=south,inner sep=1pt,column sep=4pt},
 ymajorgrids=false,xmajorgrids=true,
 error bars/y dir=none,error bars/x dir=both,error bars/x explicit,
 error bars/error mark options={rotate=0,mark size=1.5pt,line width=0.55pt}]
\foreach \yb/\xl/\xr in {4/39.00/46.07,3/38.87/47.53,2/37.87/40.73,1/39.27/41.60}{
\edef\next{\noexpand\draw[FigOutline,line width=0.6pt] (axis cs:\xl,\yb+0.10)--(axis cs:\xr,\yb-0.10);}\next}
\addplot+[basecurve,only marks] coordinates {(39.00,4.10)+-(1.25,0) (38.87,3.10)+-(1.67,0) (37.87,2.10)+-(1.70,0) (39.27,1.10)+-(2.00,0)};
\addlegendentry{LoRA}
\addplot+[ourscurve,only marks] coordinates {(46.07,3.90)+-(1.17,0) (47.53,2.90)+-(2.61,0) (40.73,1.90)+-(1.42,0) (41.60,0.90)+-(1.39,0)};
\addlegendentry{Loop Dropout}
\end{axis}
\end{tikzpicture}
\caption{MATH-500}
\label{fig:rank_math}
\end{subfigure}
\caption{\textbf{Mathematical accuracy across adapter ranks.}
Ouro-1.4B after MetaMath-GSM fine-tuning. Markers and horizontal bars show mean $\pm$ SD over three seeds. Both methods use the same five-rate search budget at each rank, with $\alpha/r=2$; \paperref{Table}{tab:rank_tuned} gives the values.}
\label{fig:rank_overview}
\end{figure}

\Needspace{25\baselineskip}
\paragraph{Zero-shot generalization to deeper recurrence.}
\label{sec:depth_sensitivity}

We evaluate zero-shot generalization beyond the training depth by applying saved adapters at additional loops without further fine-tuning.

\begin{wrapfigure}{r}{0.50\textwidth}
\centering
\captionsetup{justification=raggedright,singlelinecheck=false}
\begin{tikzpicture}[x=1.08cm,y=1.08cm,text=FigInk]
\foreach \x/\y/\p/\val/\tc/\eq in {
  0/2/37.91/+1.90/FigInk/1, 1/2/59.64/+2.98/white/0, 2/2/95.02/+4.75/white/0,
  0/1/11.62/+0.58/FigInk/0, 1/1/21.23/+1.06/FigInk/1, 2/1/54.59/+2.73/FigInk/0,
  0/0/61.66/+3.08/white/0, 1/0/35.89/+1.79/FigInk/0, 2/0/10.11/+0.51/FigInk/1} {
  \fill[FigHeatHi!\p!FigHeatLo] (\x,\y) rectangle (\x+1,\y+1);
  \draw[white,line width=1.6pt] (\x,\y) rectangle (\x+1,\y+1);
  \ifnum\eq=1
    \draw[FigInk,line width=0.6pt] (\x+0.03,\y+0.03) rectangle (\x+0.97,\y+0.97);
  \fi
  \node[text=\tc,font=\small] at (\x+0.5,\y+0.5) {$\val$};
}
\foreach \y/\lab in {2/4,1/6,0/8} {
  \node[anchor=east,font=\small] at (-0.10,\y+0.5) {\lab};
}
\foreach \x/\lab in {0/4,1/6,2/8} {
  \node[anchor=north,font=\small] at (\x+0.5,-0.06) {\lab};
}
\node[anchor=north,font=\small] at (1.5,-0.39) {Evaluation Depth};
\node[rotate=90,font=\small] at (-0.62,1.5) {Training Depth};
\shade[bottom color=FigHeatLo,top color=FigHeatHi] (3.22,0) rectangle (3.40,3);
\node[anchor=west,font=\footnotesize] at (3.44,0) {0};
\node[anchor=west,font=\footnotesize] at (3.44,3) {5};
\end{tikzpicture}
\caption{\textbf{Depth transfer on Ouro.} Loop Dropout gains over LoRA on GSM8K (percentage points). Outlined cells match training and evaluation depth.}
\label{fig:depth_sensitivity}
\label{fig:depth_gains}
\end{wrapfigure}
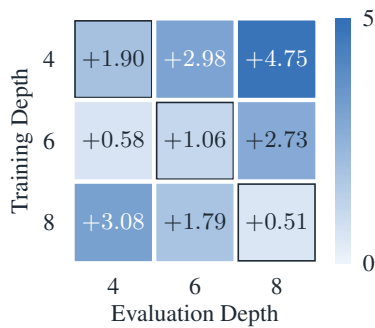

For adapters trained at four loops, \method exceeds LoRA by 2.98 and 4.75 percentage points on GSM8K when evaluated at six and eight loops, respectively.
At eight loops, \method achieves 75.6\% accuracy, compared with 70.8\% for LoRA.
\paperref{Figure}{fig:depth_sensitivity} shows the advantage growing from 1.90 points at the training depth of four loops to 4.75 points at eight loops.
Adapters trained at six loops also give a 2.73-point gain over LoRA when evaluated at eight.

\method achieves higher mean accuracy in all nine training--evaluation depth pairs, covering both matched and shifted recurrence depths.
These results connect training across application patterns with stronger generalization beyond the depth used in fine-tuning.
\paperref{Appendix}{app:depth_matrix} gives the absolute accuracies, uncertainty and evaluation-depth curves.
\par\WFclear

\FloatBarrier
\section{Conclusion}
\label{sec:conclusion}

Adapting looped language models requires a shared update that remains effective as hidden states evolve across the recurrence.
We identify a pronounced late-loop bias in standard LoRA fine-tuning and introduce \method to promote adaptation across loop positions.
The method couples stochastic masking of complete adapter applications with inverse-survival rescaling, exposing the shared update to varying application patterns while preserving its expected strength.
Every backbone loop remains active, and inference follows standard LoRA without additional trainable parameters or computation.

Experiments demonstrate improved mathematical adaptation and transfer across model sizes, adapter ranks and training recipes, with the largest gains on transfer to competition mathematics.
Comparisons with LoRA variants, adapter regularizers and matched perturbations support the proposed training design.
Further analysis shows stronger early-loop adaptation and an advantage over LoRA when adapters are evaluated beyond their training depth.
These findings highlight effective adaptation across loops as a key consideration for fine-tuning looped language models, and establish \method as a simple way to train reusable shared updates.

\subsection*{AI Use Statement}
\label{sec:ai_use}

We used ChatGPT-6 Astra and Claude Fable 5.1 to assist with writing and to identify potentially relevant work.
All details were manually verified by the authors, who take full responsibility for the content of the paper.

\subsection*{Ethics Statement}
\label{sec:ethics}

This study uses publicly released models and existing benchmark datasets and does not recruit human participants or collect new user data.
The adapted models inherit risks associated with their pretrained weights and fine-tuning data.
Deployment requires application-specific reliability and safety evaluation beyond the benchmarks studied here.

\Needspace{7\baselineskip}
\pdfbookmark[1]{\refname}{paper-references}
\bibliographystyle{iclr2027_conference}
\bibliography{reference}

\clearpage
\appendix
\pdfbookmark[0]{Appendix}{paper-appendix}
\section*{Appendix}
\section{Related Work}
\label{app:related_work}

\subsection{Parameter Sharing and Recurrent Computation}

\paragraph{Architectures and latent reasoning.}
Universal Transformers, recurrently stacked translation models and ALBERT reuse parameters across depth~\citep{dehghani2019universal,dabre2019recurrent,lan2020albert}.
Deep equilibrium models instead define representations through a fixed point of a shared transformation~\citep{bai2019deq}.
Huginn and Ouro bring recurrent depth to pretrained language models~\citep{geiping2025huginn,zhu2025ouro}, while theoretical studies characterize the expressive benefits of looping and loop-index encoding~\citep{saunshi2025looped,xu2025expressive}.
Recurrence can also be introduced into existing pretrained models: \citet{mcleish2025retrofitting} use a curriculum of recurrence depths, and LoopUS combines block decomposition with selective gates, random deep supervision and adaptive exiting~\citep{park2026loopus}.
Coconut provides a related form of latent computation by feeding a model's final hidden state back as the next input embedding~\citep{hao2025coconut}.
These approaches establish several ways to reuse computation; our study takes the pretrained recurrence as given and trains a shared downstream update within it.

\paragraph{Adaptive depth and recurrent training.}
Adaptive Computation Time and PonderNet learn how much computation to perform before producing a prediction~\citep{graves2016act,banino2021pondernet}.
Depth-Adaptive Transformers make predictions at different layers, using untied layers rather than repeatedly applying one block~\citep{elbayad2020depth}.
For looped models, LoopFormer aligns trajectories of different lengths through shortcut consistency, while Think-at-Hard combines selective iteration with depth-aware LoRA modules for hard-token refinement~\citep{jeddi2026loopformer,fu2025thinkathard}.
These choices alter the computation budget, the depth-dependent transformation, or both.
Loop Dropout keeps every backbone loop active and uses the same adapter at every loop during inference; its randomization concerns the adapter's training-time participation.

Recent work also changes the training of recurrent dynamics directly.
STARS combines random loop sampling with Jacobian spectral-radius regularization~\citep{yang2026stars}.
\citet{viakhirev2026shallow} relate depth extrapolation to finite-time dynamics, studying fixed-point training on small reasoners and a separate latent-anchoring LoRA objective for Huginn.
LoopRPT applies reinforcement signals to latent steps using an exponential-moving-average teacher and noisy latent rollouts~\citep{tang2026looprpt}.
Our training rule uses the ordinary target-token loss and randomizes complete applications of a shared low-rank update, without an auxiliary state loss or teacher.
Our depth-generalization experiments evaluate the learned update at recurrence depths beyond those used in fine-tuning.

\subsection{Parameter-Efficient Adaptation}

\paragraph{What is trained.}
Parameter-efficient fine-tuning can add bottleneck adapters~\citep{houlsby2019adapters}, optimize continuous prefixes or input prompts~\citep{li2021prefix,lester2021prompt}, or learn vectors that scale activations~\citep{liu2022ia3}.
LLM-Adapters evaluates several adapter families for language-model reasoning tasks~\citep{hu2023llmadapters}.
LoRA represents weight updates with trainable low-rank factors~\citep{hu2022lora}, allowing the learned update to be merged into the frozen weights.
AdaLoRA distributes a limited rank budget according to the importance of weight updates, QLoRA trains adapters through a quantized frozen backbone, and VeRA learns scaling vectors over shared frozen random factors~\citep{zhang2023adalora,dettmers2023qlora,kopiczko2024vera}.
These methods change the parameter or memory budget of adaptation.
We study how an adapter's update is trained across repeated applications, and evaluate both LoRA and the DoRA/VeRA parameterizations in \paperref{Appendix}{app:adapter_compatibility}.

\paragraph{How the factors are optimized.}
LoRA+ assigns different learning rates to the two factors, rsLoRA changes their rank-dependent scaling, and DoRA separates weight magnitude and direction~\citep{hayou2024loraplus,kalajdzievski2023rslora,liu2024dora}.
PiSSA initializes trainable factors from principal singular components of pretrained weights, whereas LoRA-GA uses gradient information to initialize the adapter~\citep{meng2024pissa,wang2024loraga}.
LoRA-Pro modifies the low-rank optimization to approximate full fine-tuning updates~\citep{wang2025lorapro}.
LoRA and full fine-tuning can also differ in fitting and retention~\citep{biderman2024lora}, while learning-rate sensitivity can change comparisons among adapters~\citep{lee2026lrmatters}.
Our experiments compare learning rates, ranks, independent per-loop factors and LoRA variants.
Loop Dropout retains the LoRA parameterization and couples masking of shared adapter applications with inverse-survival rescaling.

\subsection{Stochastic Regularization}
\label{app:stochastic_regularization}

\paragraph{The object and granularity of masking.}
Dropout masks activations and DropConnect masks individual weights~\citep{srivastava2014dropout,wan2013dropconnect}.
Stochastic depth removes residual branches during training, while LayerDrop applies structured layer removal to Transformers and supports extracting shallower networks for inference~\citep{huang2016stochastic,fan2020layerdrop}.
Loop Dropout masks the adaptation branch at a loop while retaining the complete pretrained transformation.
Thus the model's depth stays fixed even when the learned update is absent from some loops.
The Module-wise and Averaged gates baselines in \paperref{Table}{tab:mask_controls} test whether coupling adapted modules through one loop-level mask matters beyond perturbing individual modules or changing their common update strength.

Mask dependence across repeated computation is also a central issue in recurrent dropout.
\citet{gal2016recurrent} reuse a dropout mask across sequence time steps, and \citet{semeniuta2016recurrent} drop recurrent candidate updates in a way designed to preserve long-term memory.
Our recurrence is over depth for the same token sequence.
We sample independently across loops, share each mask across modules and token positions, and apply it only to the fine-tuning update.
This makes the mask's scope and dependence structure different from dropout on the recurrent hidden state or on the full recurrent weights.

\paragraph{Regularizing pretrained adapters.}
Mixout regularizes adaptation toward pretrained parameters~\citep{lee2020mixout}.
AdapterDrop removes adapter layers, Adapters Strike Back applies stochastic depth to adapter branches, and LoRA Dropout introduces sparsity into low-rank updates~\citep{ruckle2021adapterdrop,steitz2024adapters,lin2024loradropout}.
CoTo switches adapters at individual Transformer layers and progressively increases their activation probabilities~\citep{zhuang2025coto}; in our looped implementation, a layer's switch is reused across loops.
Loop Dropout instead switches all adapted modules off together at independently selected loop positions, while every backbone loop remains active.
Earlier applications are thus trained to remain useful when later applications are absent, while later applications encounter different inputs as earlier applications are switched on or off.
The shared update learns to work across different combinations of recurrent positions.
Inverse-survival rescaling keeps its expected training strength consistent with standard inference, where all applications are active.
The comparison with CoTo in \paperref{Table}{tab:matched_expansion} and the masking-granularity comparison in \paperref{Table}{tab:mask_controls} support this application-level design.
Untying the per-loop adapters changes which parameters each mask exposes to training, which motivates our equal-rank and equal-parameter comparisons.

\paragraph{Noise and iterative state perturbations.}
Classical noise regularization and analyses of dropout connect small perturbations to local sensitivity penalties~\citep{bishop1995noise,wager2013dropout}.
Our gate-space expansion has this interpretation locally, whereas the finite-mask mixture in \papereqref{Eq.}{eq:mixture} is exact.
Diffusion models provide another setting in which parameters are reused over evolving states: TimeStep Master allocates timestep LoRA experts, and T-LoRA adapts the rank to the denoising timestep~\citep{zhuang2025tsm,soboleva2025tlora}.
Input perturbation changes the states seen during diffusion training to address a training--sampling mismatch~\citep{ning2023input}.
In our unrolled computation, removing an earlier application changes the model-produced state received by later applications of the same update.
The matched-noise and training-rescaling controls test specific alternatives directly.

\subsection{Supervision and Evaluation}

Instruction adaptation has been studied through human demonstrations and preference feedback~\citep{ouyang2022instructgpt}, model-generated instructions~\citep{wang2023selfinstruct}, and diverse task mixtures~\citep{longpre2023flan,ivison2023tulu2}.
For mathematical adaptation, MetaMath constructs additional training questions from existing mathematical data~\citep{yu2024metamath}.
Our two settings use existing supervised recipes and vary the update training rule within each recipe.
The evaluation keeps mathematical accuracy~\citep{cobbe2021gsm8k,hendrycks2021math,lightman2024verify} separate from code correctness, truthfulness and instruction following~\citep{chen2021codex,liu2023evalplus,austin2021program,lin2021truthfulqa,zhou2023ifeval}.
\paperref{Appendix}{app:prompts} specifies the prompts, parsers and scoring rules for each setting.

\section{Update Scaling and a Local Curvature View}
\label{app:rescale_analysis}

\subsection{First-Order Consistency Through the Recurrence}
\label{app:first_order_rescale}

\begin{proof}[Proof of \paperref{Lemma}{lem:first_order_rescale}]
Fix the input, $\Delta$, $K$ and $q$ as in the lemma, and write
\begin{equation}
h_t^\eta(g)=F\bigl(h_{t-1}^\eta(g);W+\eta g_t\Delta\bigr),
\qquad z_\eta(g)=\mathrm{Head}\bigl(h_K^\eta(g)\bigr),
\end{equation}
with $h_0$ independent of $\eta$ and $g$.
At $\eta=0$, all gate vectors give the same frozen trajectory $h_t^0$ and logits $z_0$.
Define the state Jacobian $J_t=D_hF(h_{t-1}^0;W)$, the directional weight derivative $d_t=D_WF(h_{t-1}^0;W)[\Delta]$, and the head Jacobian $P=D\mathrm{Head}(h_K^0)$.
Differentiating the recurrence at $\eta=0$ gives
\begin{equation}
\dot h_t(g)=J_t\dot h_{t-1}(g)+g_t d_t,
\qquad \dot h_0(g)=0.
\end{equation}
Unrolling this identity yields
\begin{equation}
\left.\frac{\partial z_\eta(g)}{\partial\eta}\right|_{\eta=0}
=\sum_{t=1}^{K}g_t u_t,
\qquad u_t=PJ_K\cdots J_{t+1}d_t,
\label{eq:propagated_update}
\end{equation}
where the product is the identity for $t=K$.
The vector $u_t$ is the final-logit response to inserting the shared update at loop $t$, including its propagation through the remaining backbone loops.
The same $\Delta$ appears in every $d_t$, while the state and propagation Jacobians can vary with $t$.

The smoothness assumptions and finite depth make the composed output twice continuously differentiable near the frozen computation.
Consequently,
\begin{equation}
z_\eta(g)=z_0+\eta\sum_{t=1}^{K}g_tu_t+O(\eta^2).
\end{equation}
For fixed $q>0$, the mask support is finite, so the remainder is uniform over the gate vectors $b/q$, $b$ and $\ones$.
Taking expectations and using $\E[b_t/q]=1$ and $\E[b_t]=q$ gives
\begin{equation}
\begin{aligned}
\E_b\,z_\eta(b/q)&=z_0+\eta\sum_{t=1}^{K}u_t+O(\eta^2),\\
\E_b\,z_\eta(b)&=z_0+q\eta\sum_{t=1}^{K}u_t+O(\eta^2).
\end{aligned}
\end{equation}
Comparing these expressions with the expansions of $z_\eta(\ones)$ and $z_{q\eta}(\ones)$ proves \papereqref{Eq.}{eq:first_order_rescale}.

An explicit remainder bound follows by writing $Z(a)$ for the final logits when loop $t$ uses $W+a_t\Delta$, so that $z_\eta(g)=Z(\eta g)$.
Choose a sufficiently small closed ball $U$ about zero within the region of twice continuous differentiability.
There is a finite $M$ such that $\|D^2Z(a)[v,v]\|_2\leq M\|v\|_2^2$ for all $a\in U$ and vectors $v$.
For $\eta$ small enough that the masked coefficients and their means lie in $U$, Taylor expansion around each mean cancels the expected linear term and gives
\begin{equation}
\begin{aligned}
\bigl\|\E_b\,z_\eta(b/q)-z_\eta(\ones)\bigr\|_2
&\leq\frac{M\eta^2}{2}\E\|b/q-\ones\|_2^2
=\frac{MKp}{2q}\eta^2,\\
\bigl\|\E_b\,z_\eta(b)-z_{q\eta}(\ones)\bigr\|_2
&\leq\frac{M\eta^2}{2}\E\|b-q\ones\|_2^2
=\frac{MKpq}{2}\eta^2.
\end{aligned}
\label{eq:rescale_remainder_bound}
\end{equation}
The constants can depend on the fixed input, update and depth; the expansion holds with $q$ fixed.
The bound controls the nonlinear remainder, while \papereqref{Eq.}{eq:propagated_update} identifies the first-order adaptation preserved by inverse-survival rescaling.
\end{proof}

\subsection{Exact Relations}

For a fixed example and adapter $\Delta$, let $\ell_\Delta(g)$ denote the loss under gates $g\in\R^K$.
Loop Dropout uses $g_t=b_t/q$ with independent $b_t\sim\Bernoulli(q)$ and $q=1-p$, so that $\E[g]=\ones$ and $\operatorname{Cov}(g)=\frac{p}{q}I$ as in \papereqref{Eq.}{eq:gate_decomposition}.
\papereqref{Equation}{eq:mixture} gives the exact expectation of the loss over the finite set of masks.
At $p=0.5$ and $K=4$, it averages uniformly over 16 patterns, from no active update to all four applications.
The all-zero pattern contributes a loss independent of the adapter; the other patterns provide different training contexts for its active applications.
The total gate sum $D=\sum_tg_t$ satisfies $qD\sim\operatorname{Binomial}(K,q)$, with $\E[D]=K$ and $\operatorname{Var}(D)=Kp/q$.
Thus the expected number of applications, counted with their scale, equals the $K$ applications used at inference.
At $p=1/2$ and $K=4$, its possible values are $0,2,4,6,8$ with probabilities $(1,4,6,4,1)/16$.
The mean-preserving property is defined in parameter space; the recurrence retains its nonlinear dependence on the gate vector.

The unscaled rule has mean $q\ones$ and covariance $pqI$.
For any adapter $\Delta$ and mask $b\in\{0,1\}^K$,
\begin{equation}
\ell_{\Delta/q}(b)=\ell_{\Delta}(b/q),
\label{eq:reparam}
\end{equation}
so the two rules parameterize the same family of masked computations.
Let $\Delta_N$ and $\Delta_U$ denote the unscaled and rescaled adapters.
By \papereqref{Eq.}{eq:reparam}, the corresponding parameters $\Delta_N=\Delta_U/q$ represent the same distribution of masked forward computations, $\ell_{\Delta_N}(b)=\ell_{\Delta_U}(b/q)$ for every $b$.
This identity describes corresponding parameterizations; optimization additionally depends on the chosen factor initialization and learning rates.
At unit-gate inference, the two corresponding updates differ by $1/q$; evaluating $\Delta_N$ with gate $q$ restores the same effective update as evaluating $\Delta_U$ with gate one.
The constant can thus be applied during training or folded into the adapter at inference; \method applies it during training so that inference uses the unit gate of standard LoRA without a calibration constant, whereas the unscaled control in \paperref{Table}{tab:mask_controls} trains and evaluates without it.
Unscaled training samples its own inference gate with probability $q^K$; rescaled training instead centers its gate distribution at the inference gate.

\subsection{Second-Order Approximation}

To interpret sensitivity near the default inference gate, write $g=\ones+\varepsilon$, with $\E[\varepsilon]=0$ and covariance $C$.
A Taylor expansion gives
\begin{equation}
\E\,\ell_\Delta(\ones+\varepsilon)
=\ell_\Delta(\ones)
+\frac12\tr\!\left(C\nabla_g^2\ell_\Delta(\ones)\right)+R,
\label{eq:expansion}
\end{equation}
where $R$ collects the expected higher-order terms.
Let $v_t=\left.\partial z/\partial g_t\right|_{g=\ones}$ be the local logit sensitivity to gate $t$, let $J=[v_1,\ldots,v_K]$, and let $H$ be the cross-entropy Hessian with respect to the stacked target-token logits, with the same averaging convention as the loss.
The gate Hessian decomposes as
\begin{equation}
\nabla_g^2\ell_\Delta(\ones)=J^\top HJ+\sum_i\frac{\partial\ell_\Delta}{\partial z_i}\,\nabla_g^2 z_i,
\end{equation}
and substituting $C=\frac{p}{q}I$ into \papereqref{Eq.}{eq:expansion} gives \papereqref{Eq.}{eq:penalty}, with $R$ collecting the second summand and the higher-order terms.
Keeping only the Gauss--Newton part $J^\top HJ$ is the gate-space analogue of local noise-regularization analyses and the adaptive-regularization view of dropout~\citep{bishop1995noise,wager2013dropout}.
At $p=1/2$ the gate perturbations have unit magnitude, so \papereqref{Eq.}{eq:penalty} is a local description of the objective, while \papereqref{Eq.}{eq:mixture} gives its exact form.

\subsection{What the Controls Match}

Let $P_\parallel=K^{-1}\ones\ones^\top$ and $P_\perp=I-P_\parallel$.
The Averaged gates baseline replaces each sampled gate by $D/K$ and hence has covariance $\frac{p}{q}P_\parallel$.
Retaining exactly $m=qK$ uniformly sampled loops with gate $1/q$ gives covariance $\frac{p}{q}\frac{K}{K-1}P_\perp$.
Independent Bernoulli masks retain both components.
These covariance relations describe which local sensitivity directions enter \papereqref{Eq.}{eq:expansion}; the distributions also differ in their support and higher moments.
The exact mixture incorporates all of these distributional properties.

At unit strength, the parallel-noise control has the same mean and covariance as Loop Dropout, but includes negative and unbounded gates and lacks a point mass at zero.
The comparison in \paperref{Table}{tab:mask_controls} tests this continuous perturbation against finite masking with the same first two gate moments.

Finally, the unscaled expansion is centered at $q\ones$ and has coefficient $pq/2$.
Under the corresponding parameterizations in \papereqref{Eq.}{eq:reparam}, the derivative scaling compensates for the difference between this coefficient and $p/(2q)$.
The expansion therefore describes the local objective in each parameterization, with its sensitivities evaluated at the stated adapter and gate.

\section{Experimental Setup and Evaluation}
\label{app:experimental_details}

\subsection{Adapter and Mask Implementation}
\label{app:implementation}

\paragraph{Adapter placement.}
The Ouro recurrent block contains decoder layers with four attention projections (query, key, value and output) and three feed-forward projections (gate, up and down).
Unless a control specifies otherwise, all seven projections receive shared LoRA factors, while embeddings, normalization and the output head remain frozen.
The default rank is 16 and $\alpha=32$, giving 15{,}138{,}816 trainable parameters for Ouro-1.4B and twice that number for Ouro-2.6B.
The rank sweep keeps $\alpha/r=2$; rsLoRA instead uses $\alpha/\sqrt r$.
We initialize $A$ with Kaiming-uniform weights and $B$ to zero, store adapter parameters in float32, and cast their forward computation to the activation dtype.

\paragraph{Mask sampling.}
At each micro-batch, the implementation draws a $K\times B$ mask for the $B$ examples and exposes the current loop to every adapted module.
Each module multiplies its adapter output by the corresponding gate, broadcasting over token positions.
A dropped application therefore removes all LoRA contributions from one loop while retaining its frozen operations.
The gate is one at every loop during evaluation.
A unit-gate shared adapter can be merged into the pretrained matrices in the usual way.

\paragraph{Sharing, scale and granularity controls.}
The unscaled control uses $g=b$.
Per-loop adapters use a separate pair of factors for each loop, either at rank 16 per loop or at rank 4 to match the shared rank-16 parameter count.
The Averaged gates baseline draws the same mask as Loop Dropout and sets every gate of an example to $K^{-1}\sum_tg_t$.
It preserves the sampled gate sum, but removes loop-to-loop variation; the all-zero mask disables every application simultaneously.
Module-wise masking draws independent masks for each adapted linear map and loop.
The standard input-dropout baseline applies probability 0.1 dropout to the adapter input; it is distinct from the sparsity method named LoRA Dropout in \citet{lin2024loradropout}.

\paragraph{Matched noise controls.}
Both noise controls draw an event $a_t\sim\Bernoulli(p)$ per example and loop, with event draws paired to the removal events of Loop Dropout.
Parallel noise uses
\begin{equation}
g_t=1+a_t c\xi_t/\sqrt q,\qquad \xi_t\sim\mathcal N(0,1),
\end{equation}
where the scalar is shared across modules and tokens at that loop.
At $c=1$, its mean and covariance equal those of the Loop Dropout gate.
Low-rank weight noise instead perturbs each adapted matrix by
\begin{equation}
\widetilde\Delta_i=\Delta_i+
\frac{a_t c\,\operatorname{stopgrad}(\|\Delta_i\|_F)}{\sqrt{q\,d_{\mathrm{out}}d_{\mathrm{in}}r_n}}
G_{\mathrm{out}}G_{\mathrm{in}},
\end{equation}
where the two independent standard-Gaussian factors have shapes $d_{\mathrm{out}}\times r_n$ and $r_n\times d_{\mathrm{in}}$, with $r_n=16$.
Its expected perturbation energy is $c^2(p/q)\|\Delta_i\|_F^2$, matching Loop Dropout at $c=1$.
This is a random low-rank perturbation, not independent dense Gaussian noise on every weight.

\paragraph{Other masks and random-number streams.}
The structured variants use nonuniform loop probabilities, exactly sized subsets, random prefixes or suffixes, training schedules, or sensitivity-dependent probabilities, with rescaling by each loop's retention probability.
Initialization and data order are fixed before masks are sampled, and masking and noise use separate random streams, so paired same-rank comparisons start from identical adapters as specified in \paperref{Appendix}{app:hyperparameters}.

\subsection{Models, Data and Splits}
\label{app:models_data}

\paragraph{Models.}
Ouro-1.4B and Ouro-2.6B have recurrent blocks of 24 and 48 decoder layers, respectively, with hidden size 2048, 16 attention heads and vocabulary size 49{,}152.
We use the base checkpoints with four loops unless a depth experiment specifies otherwise.
The backbone uses bfloat16 and scaled-dot-product attention, and adaptive exit gates are disabled.
Mathematical prompts do not use a chat template; instruction tuning uses the T\"ulu~2 turn format.
Batched generation uses left padding with an attention mask that covers the recurrent cache.

\paragraph{MetaMath-GSM-100k.}
Following the data-processing recipe of LoRA-Pro~\citep{wang2025lorapro}, we retain MetaMathQA examples whose type contains ``GSM'', remove examples of at least 512 tokens under the Ouro tokenizer, and take the first 100{,}000 in source order for training.
A further 500 GSM-type MetaMathQA examples, disjoint from the training examples, form the development split used for hyperparameter selection.
Inputs use the Alpaca instruction template without an input field; targets are the reference response followed by the end-of-sequence token.
We use the Ouro tokenizer and the adapter configuration and learning rates specified below.

\paragraph{Direct GSM8K.}
We train on 6{,}973 of the 7{,}473 training problems and hold out the remaining 500.
Targets contain the reference solution with calculator annotations removed, retain the final ``\texttt{\#\#\#\# N}'' line, and end with the end-of-sequence token.
This smaller recipe supports the broader learning-rate, mask and depth studies.

\paragraph{T\"ulu~2 instruction tuning.}
A fixed 100k-example draw from the mixture yields 98{,}415 trainable examples after filtering.
Only assistant tokens enter the instruction-tuning loss; the mathematical recipes likewise use target-only loss.
\paperref{Table}{tab:recipes} summarizes the three recipes.

\begin{table}[h]
\centering
\small
\caption{\textbf{Fine-tuning recipes.} Effective batch = micro-batch $\times$ gradient accumulation. The adapter recipes use AdamW~\citep{loshchilov2019adamw} ($\beta_1=0.9$, $\beta_2=0.999$), zero weight decay, gradient clipping at 1.0, a cosine schedule decaying to 10\% of the peak learning rate, and the final checkpoint.}
\label{tab:recipes}
\resizebox{\textwidth}{!}{%
\begin{tabular}{lccc}
\toprule
 & MetaMath-GSM-100k & GSM8K & T\"ulu~2-100k \\
\midrule
Training examples & 100{,}000 & 6{,}973 & 98{,}415 \\
Epochs / optimizer steps & 1 / 3{,}125 & 2 / $\approx$1{,}744 & 1 / 769 \\
Effective batch & $8\times4=32$ & $8\times1=8$ & $8\times16=128$ \\
Maximum sequence length & 1{,}024 & 512 & 2{,}048 \\
Warm-up & 3\% & 5\% & 3\% \\
Learning rates & See below & 3e-5 to 6e-4 & 1e-4 \\
Seeds & 101--103 & 0--4 / 101--103 & 101--103 \\
Evaluation & MetaMath evaluator & lm-eval (0-shot) & \paperref{Appendix}{app:tulu_returned} \\
\bottomrule
\end{tabular}}
\end{table}

\subsection{Training Hyperparameters}
\label{app:hyperparameters}

\paragraph{Adapters and seeds.}
Default adapters have rank 16 and $\alpha=32$; the independently tuned rank study uses $r\in\{4,16,64,128\}$ with $\alpha/r=2$.
LoRA+ uses a fourfold learning-rate multiplier for $B$ relative to $A$.
Studies with seeds 101--103 set the seed before attaching adapters, so paired same-rank methods share initial adapter weights and data order; the GSM8K learning-rate, depth and mask studies with seeds 0--4 use independent initializations.
Each study block is reported separately.

\paragraph{MetaMath learning rates.}
For the main mathematical comparisons and the independently tuned sharing and rank studies, we tune the learning rate over $\{1.25,2.5,5,10,20\}\times10^{-5}$, or half these rates for the LoRA+ $A$ factor.
Each method chooses its rate by greedy-generation accuracy on the 500-example development split of \paperref{Appendix}{app:models_data} with training seed 101, resolving ties toward the smaller rate.
For the three-seed comparisons, seeds 102 and 103 are then trained at that rate.
Per-seed configurations and results are provided in \paperref{Appendices}{app:per_seed}, \papernumref{app:comparison_completion} and~\papernumref{app:rank_tuned}.

\paperref{Tables}{tab:control_seeds} and~\papernumref{tab:returned_seeds} specify the configurations for the masking and noise comparisons, and \paperref{Table}{tab:rank_results} reports the additional rank study.
The noise baselines use strength $c=1$ in \paperref{Table}{tab:mask_controls}; \paperref{Table}{tab:returned_seeds} also includes parallel noise with $c=2$.

\paragraph{Sharing and masking.}
\label{app:sharing_settings}
\paperref{Table}{tab:sharing_mask} compares training seeds 101--103 for every configuration.
We tune the learning rate over $\{1.25,2.5,5,10,20\}\times10^{-5}$ independently for each parameterization with and without \method, using the development procedure above.
The shared-adapter rows are the rank-16 results in \paperref{Table}{tab:math_main}.
All masked configurations use $p=0.5$ and inverse-survival rescaling.
The diagnostics in \paperref{Figure}{fig:early_adaptation} use shared rank-16 checkpoints at these independently chosen rates, with seeds 101--103.

\paragraph{Other recipes.}
For direct GSM8K, we tune the learning rate over $\{3\times10^{-5},10^{-4},2\times10^{-4},3\times10^{-4},6\times10^{-4}\}$; \paperref{Table}{tab:gsm8k_scale} reports the seed counts and results.
\paperref{Table}{tab:recipes} specifies the T\"ulu~2 training recipe.
All recipes use the final checkpoint.

\paragraph{Other looped transformers.}
\label{app:family_settings}
\paperref{Table}{tab:family_comparison} uses rank-16 adapters with $\alpha=32$ and the MetaMath-GSM-100k recipe.
LoopUS-Qwen3-4B uses eight loops and training seeds 501--503.
For both methods, we re-evaluate the seed-501 checkpoints with the same pipeline as seeds 502 and 503; \paperref{Table}{tab:family_comparison} averages these three evaluations.
We select $p=0.3$ for \method on the development split from $\{0.1,0.2,0.3,0.4,0.5\}$.
LoopCoder-V2 uses two loops, $p=0.5$ for \method and training seeds 101--103; we tune the learning rate over $\{1.25,2.5,5,10,20\}\times10^{-5}$ independently for each method, using the development procedure above.
Both methods otherwise use the same settings within each model, with $p=0$ for LoRA and all adapter applications enabled at inference.

\subsection{Compute and Timing}
\label{app:compute}

Training uses a single H100, H200 or B200 GPU and bfloat16, with Transformers 5.3.0 for LoopUS and 4.57.6 for the other models.
The cost comparison in \paperref{Table}{tab:matched_expansion} uses H100 80GB measurements with PyTorch 2.8.0, scaled-dot-product attention and gradient checkpointing disabled for every reported method.
All five methods use the same 100k examples, 3,125 optimizer steps and rank-16 adapters.
Training time and peak memory are means over three training seeds.
For each seed, training time covers one selected-configuration training run and excludes development search, generation, scoring and queue time.
\paperref{Table}{tab:training_cost_seeds} reports the individual timings.
The Ouro-2.6B five-candidate comparison uses B200 GPUs with PyTorch 2.7.1; its timing is kept separate from the H100 comparison.
Direct GSM8K training at depth four takes roughly 12--14 minutes for Ouro-1.4B and 25--32 minutes for Ouro-2.6B.

\begin{table}[htbp]
\centering
\small
\caption{\textbf{Training time across seeds on H100 80GB.} Times are in minutes for seeds 101 / 102 / 103, with their mean and sample SD. Peak memory is the mean over the same runs.}
\label{tab:training_cost_seeds}
\begin{tabular}{lccc}
\toprule
Method & Per-seed time & Mean $\pm$ SD & Memory (GB) \\
\midrule
LoRA & 111.93 / 106.74 / 113.28 & 110.65 $\pm$ 3.45 & 45.36 \\
LoRA+ & 111.20 / 107.74 / 119.40 & 112.78 $\pm$ 5.99 & 45.40 \\
CoTo-on-Ouro & 94.95 / 93.32 / 94.06 & 94.11 $\pm$ 0.82 & 45.28 \\
LoRA Dropout & 442.39 / 438.34 / 453.98 & 444.91 $\pm$ 8.12 & 45.37 \\
\method & 130.86 / 128.16 / 127.65 & 128.89 $\pm$ 1.72 & 45.36 \\
\bottomrule
\end{tabular}
\end{table}

\subsection{Prompts, Decoding and Scoring}
\label{app:prompts}

\paragraph{MetaMath evaluation.}
The main mathematical evaluation follows the MetaMath scripts at revision \texttt{fe667b1}, using the Alpaca instruction prompt~\citep{taori2023alpaca} and the response prefix ``\texttt{Let's think step by step.}'' used in zero-shot chain-of-thought prompting~\citep{kojima2022zeroshot}.
Evaluation is zero-shot and greedy, with the reference stop strings and maximum new-token budgets of 512 for GSM8K and 2{,}048 for MATH-500.
The scorer extracts the answer after ``\texttt{The answer is:}'', then applies numeric comparison for GSM8K or the reference MATH normalization and equivalence rules.
MATH-500 is the 500-problem subset released by \citet{lightman2024verify}.

Generation uses Hugging Face Transformers~\citep{wolf2020transformers} with left-padded batches and an attention mask that covers the recurrent cache.
Every comparison uses the same prompt, decoding settings and scorer; \paperref{Appendix}{app:answer_decomposition} additionally quantifies gains on problems where both methods provide extractable answers.

\paragraph{Direct GSM8K fine-tuning.}
These experiments use the zero-shot \texttt{gsm8k} task in lm-eval-harness~\citep{gao2023lmeval}, greedy decoding and a 256-token generation cap.
Strict exact match compares the number after ``\texttt{\#\#\#\#}'' with the reference on all 1{,}319 test problems.
GSM8K-Platinum (1{,}209 relabelled problems) and GSM-Plus mini (2{,}400 perturbed problems) use the same extraction convention.
This protocol is distinct from the MetaMath recipe and its scores are not pooled with that recipe.

\paragraph{Instruction-tuning evaluation.}
\paperref{Appendix}{app:tulu_returned} specifies the reported task metrics, generation budgets and official code/IFEval scoring rules.
These use the instruction-tuned checkpoints and are kept separate from the mathematical recipe and raw-prompt base-model evaluations.

\subsection{Seeds and Uncertainty}
\label{app:uncertainty}

Main configurations use three training seeds; two learning rates of the GSM8K study use five.
Tables report the mean and sample standard deviation over training seeds, which summarize training-seed variation; paired same-rank methods share initial adapter weights and data order as specified in \paperref{Appendix}{app:hyperparameters}.

\section{Complete Benchmark Results}
\label{app:additional_results}

\FloatBarrier
\subsection{Other Looped Transformers}
\label{app:family_comparison}
We also evaluate \method on LoopUS-Qwen3-4B~\citep{park2026loopus} and LoopCoder-V2~\citep{yang2026loopcoderv2} under the MetaMath-GSM-100k recipe.
\paperref{Table}{tab:family_comparison} shows higher accuracy on both benchmarks for both models; on LoopUS-Qwen3-4B, the three-seed mean gains are 0.58 percentage points on GSM8K and 3.07 points on MATH-500.
On LoopCoder-V2, the corresponding gains are 2.12 and 3.53 points.
\paperref{Appendix}{app:family_settings} specifies their training settings.

\begin{table}[htbp]
\centering
\footnotesize
\setlength{\tabcolsep}{7pt}
\caption{\textbf{Mathematical evaluation on other looped transformers.} Accuracy in \%, mean $\pm$ SD over three training seeds for each model and method. Bold marks the higher mean within each model and benchmark.}
\label{tab:family_comparison}
\begin{tabular}{lcccc}
\toprule
& \multicolumn{2}{c}{LoopUS-Qwen3-4B} & \multicolumn{2}{c}{LoopCoder-V2} \\
\cmidrule(lr){2-3}\cmidrule(lr){4-5}
Benchmark & LoRA & \method & LoRA & \method \\
\midrule
GSM8K & 84.91 $\pm$ 0.62 & \textbf{85.49 $\pm$ 0.34} & 85.82 $\pm$ 0.35 & \textbf{87.95 $\pm$ 0.27} \\
MATH-500 & 37.87 $\pm$ 2.83 & \textbf{40.93 $\pm$ 1.80} & 41.67 $\pm$ 0.31 & \textbf{45.20 $\pm$ 2.40} \\
\bottomrule
\end{tabular}
\end{table}

\FloatBarrier
\subsection{Per-Seed Results for the Masking Comparison}
\label{app:per_seed}

\begin{table}[H]
\centering
\footnotesize
\setlength{\tabcolsep}{5pt}
\caption{\textbf{Per-seed results of the mask-control comparison} on Ouro-1.4B under the MetaMath protocol, at the learning rates shown; accuracy in \%.}
\label{tab:control_seeds}
\begin{tabular}{llcc}
\toprule
Method & LR & GSM8K seeds 101 / 102 / 103 & MATH-500 seeds 101 / 102 / 103 \\
\midrule
LoRA & 1e-4 & 84.6 / 85.0 / 86.4 & 39.2 / 42.0 / 39.2 \\
\method & 1e-4 & 87.7 / 87.0 / 88.0 & 44.0 / 45.4 / 42.6 \\
\midrule
Unscaled & 1e-4 & 84.5 / 84.7 / 84.9 & 40.6 / 40.0 / 37.4 \\
Unscaled & 2e-4 & 83.8 / 83.8 / 84.8 & 35.6 / 38.0 / 37.4 \\
Averaged gates & 1e-4 & 85.4 / 84.8 / 86.5 & 36.8 / 40.2 / 38.8 \\
Averaged gates & 2e-4 & 85.5 / 84.9 / 85.9 & 37.8 / 42.4 / 37.8 \\
Module-wise & 1e-4 & 85.0 / 85.4 / 86.2 & 40.6 / 42.0 / 40.8 \\
\bottomrule
\end{tabular}
\end{table}

\FloatBarrier
\subsection{Instruction Tuning on T\"ulu 2}
\label{app:tulu_returned}

The instruction-tuning study uses a 100k-example draw from the T\"ulu~2 mixture, yielding 98{,}415 trainable examples after filtering.
Ouro-1.4B is trained for 769 optimizer steps at context length 2{,}048 and effective batch size 128, with assistant-only loss and rank-16 adapters ($\alpha=32$).
\paperref{Table}{tab:recipes} specifies the remaining training settings.
Each method is trained with seeds 101--103, and same-seed methods share adapter initialization and data order.
\paperref{Table}{tab:tulu_all} reports all nine metrics computed for these checkpoints; \paperref{Table}{tab:tulu_main} shows four benchmarks and the six-benchmark average.
The six-benchmark average uses HumanEval+, MBPP+, MMLU, BBH, TruthfulQA MC2 and strict IFEval.

\begin{table}[htbp]
\centering
\footnotesize
\setlength{\tabcolsep}{4pt}
\caption{\textbf{All downstream metrics after T\"ulu~2 instruction tuning of Ouro-1.4B.} Rank 16; accuracy in \%, mean $\pm$ SD over three matched seeds. HumanEval+ and MBPP+ add the expanded EvalPlus tests to the original ones; IFEval reports prompt-level accuracy. The last three rows average the four code metrics, the six-benchmark set, and all nine metrics per seed; bold marks the highest average in each summary row.}
\label{tab:tulu_all}
\begin{tabular}{lcc}
\toprule
Metric & LoRA & \method \\
\midrule
MMLU (0-shot) & 68.64 $\pm$ 0.40 & 68.91 $\pm$ 0.13 \\
BBH (3-shot CoT) & 71.20 $\pm$ 0.14 & 70.78 $\pm$ 0.11 \\
TruthfulQA MC2 & 47.32 $\pm$ 1.02 & 48.19 $\pm$ 0.68 \\
HumanEval & 71.54 $\pm$ 2.14 & 72.56 $\pm$ 0.61 \\
HumanEval+ & 67.68 $\pm$ 3.05 & 69.92 $\pm$ 0.35 \\
MBPP & 75.57 $\pm$ 0.40 & 76.72 $\pm$ 0.53 \\
MBPP+ & 64.73 $\pm$ 0.55 & 64.81 $\pm$ 0.79 \\
IFEval strict & 46.33 $\pm$ 1.26 & 46.46 $\pm$ 1.02 \\
IFEval loose & 50.59 $\pm$ 1.44 & 50.40 $\pm$ 0.11 \\
\midrule
Code average (4 metrics) & 69.88 $\pm$ 1.21 & \textbf{71.00 $\pm$ 0.46} \\
Average (6 benchmarks) & 60.98 $\pm$ 0.66 & \textbf{61.51 $\pm$ 0.09} \\
Average (9 metrics) & 62.62 $\pm$ 0.78 & \textbf{63.20 $\pm$ 0.14} \\
\bottomrule
\end{tabular}
\end{table}

\paragraph{Evaluation.}
MMLU~\citep{hendrycks2021mmlu} (14{,}042 questions, 0-shot) and BBH~\citep{suzgun2023bbh} (27 tasks, 6{,}511 questions, 3-shot chain of thought) use lm-eval-harness~\citep{gao2023lmeval}; TruthfulQA MC2 averages the probability mass assigned to correct answers over 817 questions.
The code evaluations use the pinned official EvalPlus scorer~\citep{liu2023evalplus} on all 164 HumanEval~\citep{chen2021codex} and 378 MBPP~\citep{austin2021program} tasks in its evaluation sets, with one greedy generation per task capped at 1{,}024 tokens; HumanEval+ and MBPP+ require passing the additional tests as well as the original ones.
IFEval uses all 541 prompts, a 2{,}048-token cap and the pinned official evaluator with scoring seed zero, reporting prompt-level strict and loose accuracy.
Code and IFEval generation use T\"ulu user/assistant turns.
All configurations use the same environment, prompt definitions and benchmark instances.

\paragraph{Reading the table.}
Compared with LoRA, \method increases the nine-metric average from 62.62 to 63.20, the six-metric average in \paperref{Table}{tab:tulu_main} from 60.98 to 61.51, and the average over four code metrics from 69.88 to 71.00.
It improves HumanEval+, MBPP and TruthfulQA MC2 by 2.24, 1.15 and 0.87 points, respectively; the MBPP improvement holds in every seed.
On the remaining six metrics, \method lies within about one point of LoRA.
The seed standard deviations summarize training-seed variation, as described in \paperref{Appendix}{app:uncertainty}.

\paragraph{Single-loop adaptation after instruction tuning.}
We apply the single-loop diagnostic of \paperref{Figure}{fig:teaser_reduction} to the six instruction-tuned checkpoints above.
Each evaluation enables the adapter at one loop with unit gate while retaining all four backbone loops.
We score 500 held-out T\"ulu conversations that were used neither for training nor for model selection, using assistant-only teacher-forced cross entropy.
We average the loss across conversations and report its percentage reduction relative to the frozen model.
\paperref{Table}{tab:tulu_single_loop} reports the mean and sample SD over the three training seeds; \paperref{Table}{tab:tulu_single_loop_sources} gives the breakdown across all twelve data sources.

\begin{supplementtable}[htbp]
\centering
\small
\setlength{\tabcolsep}{4pt}
\caption{\textbf{Single-loop adaptation after T\"ulu instruction tuning.} Reduction in final-loop loss relative to the frozen model, in \%, mean $\pm$ SD over three training seeds. Only the indicated adapter application is enabled; all four backbone loops run. Higher is better.}
\label{tab:tulu_single_loop}
\begin{tabular}{lcccc}
\toprule
Method & Only loop 1 & Only loop 2 & Only loop 3 & Only loop 4 \\
\midrule
LoRA & 13.07 $\pm$ 0.30 & 18.30 $\pm$ 0.11 & 18.20 $\pm$ 0.02 & 19.82 $\pm$ 0.14 \\
\method & \textbf{23.75 $\pm$ 0.76} & \textbf{29.87 $\pm$ 0.43} & \textbf{30.23 $\pm$ 0.42} & \textbf{27.70 $\pm$ 0.28} \\
\bottomrule
\end{tabular}
\end{supplementtable}

\method improves the loss reduction at every loop in every seed, with gains of 7.6--12.5 percentage points over LoRA.
Across the twelve data sources and four loop positions, it gives a larger mean reduction in 47 of the 48 comparisons.
The stronger individual applications therefore also appear after instruction tuning on a mixture of non-mathematical and mathematical tasks.

\begin{supplementtable}[htbp]
\centering
\footnotesize
\setlength{\tabcolsep}{3pt}
\caption{\textbf{Single-loop loss reduction by T\"ulu data source.} The same diagnostic as \paperref{Table}{tab:tulu_single_loop}, computed within each source. Parentheses give the number of held-out conversations; these sum to 500. Values are percentages, mean $\pm$ SD over three training seeds. Bold marks the larger mean within each source and loop.}
\label{tab:tulu_single_loop_sources}
\begin{tabular}{llcccc}
\toprule
Source ($n$) & Method & Only loop 1 & Only loop 2 & Only loop 3 & Only loop 4 \\
\midrule
Code Alpaca (29) & LoRA & 15.45 $\pm$ 0.90 & 26.13 $\pm$ 1.49 & 19.48 $\pm$ 1.63 & 38.59 $\pm$ 0.25 \\
 & \method & \textbf{35.62 $\pm$ 2.33} & \textbf{47.59 $\pm$ 0.68} & \textbf{45.03 $\pm$ 0.47} & \textbf{44.01 $\pm$ 0.33} \\
\addlinespace[2pt]
CoT (68) & LoRA & 11.34 $\pm$ 0.34 & 17.06 $\pm$ 0.47 & 20.73 $\pm$ 0.46 & 20.36 $\pm$ 0.06 \\
 & \method & \textbf{22.09 $\pm$ 0.63} & \textbf{30.27 $\pm$ 0.44} & \textbf{31.15 $\pm$ 0.20} & \textbf{25.94 $\pm$ 0.26} \\
\addlinespace[2pt]
FLAN v2 (95) & LoRA & 20.19 $\pm$ 1.19 & 27.94 $\pm$ 0.61 & 30.76 $\pm$ 0.43 & 29.77 $\pm$ 0.19 \\
 & \method & \textbf{36.68 $\pm$ 0.85} & \textbf{47.01 $\pm$ 0.92} & \textbf{49.41 $\pm$ 1.03} & \textbf{45.19 $\pm$ 0.76} \\
\addlinespace[2pt]
GPT-4 Alpaca (35) & LoRA & 11.04 $\pm$ 0.84 & 17.22 $\pm$ 0.67 & 13.78 $\pm$ 0.27 & 23.94 $\pm$ 0.27 \\
 & \method & \textbf{24.86 $\pm$ 1.84} & \textbf{29.48 $\pm$ 0.31} & \textbf{28.49 $\pm$ 0.32} & \textbf{28.25 $\pm$ 0.08} \\
\addlinespace[2pt]
LIMA (2) & LoRA & 2.23 $\pm$ 0.93 & 3.10 $\pm$ 0.64 & 1.92 $\pm$ 0.41 & \textbf{8.99 $\pm$ 0.10} \\
 & \method & \textbf{4.33 $\pm$ 0.79} & \textbf{5.46 $\pm$ 1.05} & \textbf{4.50 $\pm$ 1.26} & 8.44 $\pm$ 0.55 \\
\addlinespace[2pt]
OpenAssistant (11) & LoRA & 1.28 $\pm$ 0.12 & 1.61 $\pm$ 0.07 & 1.10 $\pm$ 0.03 & 1.79 $\pm$ 0.02 \\
 & \method & \textbf{1.88 $\pm$ 0.07} & \textbf{2.39 $\pm$ 0.12} & \textbf{2.06 $\pm$ 0.02} & \textbf{2.12 $\pm$ 0.05} \\
\addlinespace[2pt]
OpenOrca (49) & LoRA & 18.49 $\pm$ 1.20 & 21.25 $\pm$ 0.51 & 13.89 $\pm$ 0.63 & 12.94 $\pm$ 0.45 \\
 & \method & \textbf{24.65 $\pm$ 1.57} & \textbf{25.44 $\pm$ 0.62} & \textbf{23.17 $\pm$ 0.18} & \textbf{21.79 $\pm$ 0.30} \\
\addlinespace[2pt]
Evidence inference (3) & LoRA & 11.18 $\pm$ 6.86 & 27.07 $\pm$ 2.54 & 23.07 $\pm$ 1.37 & 24.40 $\pm$ 0.45 \\
 & \method & \textbf{27.18 $\pm$ 4.15} & \textbf{37.45 $\pm$ 0.95} & \textbf{29.24 $\pm$ 0.52} & \textbf{28.04 $\pm$ 0.66} \\
\addlinespace[2pt]
SciERC relation (1) & LoRA & -0.53 $\pm$ 0.81 & 12.40 $\pm$ 1.39 & 21.61 $\pm$ 1.20 & 18.35 $\pm$ 0.23 \\
 & \method & \textbf{11.29 $\pm$ 2.78} & \textbf{20.72 $\pm$ 0.55} & \textbf{22.82 $\pm$ 0.44} & \textbf{21.23 $\pm$ 0.81} \\
\addlinespace[2pt]
SciFact (1) & LoRA & 5.00 $\pm$ 6.84 & 38.58 $\pm$ 5.99 & 47.08 $\pm$ 2.15 & 72.28 $\pm$ 0.45 \\
 & \method & \textbf{36.44 $\pm$ 3.74} & \textbf{60.53 $\pm$ 6.94} & \textbf{69.35 $\pm$ 1.71} & \textbf{72.61 $\pm$ 1.24} \\
\addlinespace[2pt]
ShareGPT (168) & LoRA & 3.28 $\pm$ 0.02 & 4.39 $\pm$ 0.07 & 2.82 $\pm$ 0.17 & 4.31 $\pm$ 0.10 \\
 & \method & \textbf{5.03 $\pm$ 0.41} & \textbf{5.73 $\pm$ 0.13} & \textbf{4.83 $\pm$ 0.15} & \textbf{4.86 $\pm$ 0.12} \\
\addlinespace[2pt]
WizardLM (38) & LoRA & 6.09 $\pm$ 0.35 & 9.00 $\pm$ 0.68 & 6.31 $\pm$ 0.77 & 10.42 $\pm$ 0.13 \\
 & \method & \textbf{13.13 $\pm$ 0.93} & \textbf{14.62 $\pm$ 0.32} & \textbf{13.92 $\pm$ 0.16} & \textbf{12.77 $\pm$ 0.05} \\
\bottomrule
\end{tabular}
\end{supplementtable}

\FloatBarrier

\FloatBarrier
\section{Extended Ablations}
\label{app:extended_ablations}

\subsection{Rescaling and Dropout Probability}
The GSM8K comparison favors rescaled dropout over unscaled dropout at both tested learning rates.
At $10^{-4}$, the accuracies are 80.6\% for LoRA, 77.0\% for unscaled dropout with $p=0.5$, and 81.7\% with rescaling; \paperref{Table}{tab:mask_controls} gives the three-seed MetaMath comparison.
\paperref{Appendix}{app:rescale_analysis} relates the two training parameterizations.

A three-seed probability study obtains 81.1 $\pm$ 0.5, 81.7 $\pm$ 0.2 and 81.6 $\pm$ 1.1 at $p=0.25,0.5,0.75$, respectively, compared with 80.7 $\pm$ 0.2 for LoRA and 79.5 $\pm$ 0.5 for standard adapter input dropout.
At the higher learning rate the corresponding means are 80.0, 81.5 and 81.0, versus 78.1 for LoRA.
The middle probability gives the highest mean at both learning rates and is our default.

\subsection{Dropout Variants and Structural Controls}
\label{app:variants}
\paperref{Table}{tab:variants} compares uniform masking with alternative masks, dropout schedules and learned gates.

\begin{table}[htbp]
\centering
\footnotesize
\setlength{\tabcolsep}{4pt}
\caption{\textbf{Structured, scheduled and adaptive masks} on the GSM8K recipe (Ouro-1.4B, strict exact match, \%; two seeds, mean $\pm$ SD). Masks control adapter applications while all four backbone loops run; all variants keep the $1/P(\text{keep})$ rescaling. Bottom: five-seed follow-up of the two closest variants. No variant exceeds the uniform rule at both learning rates.}
\label{tab:variants}
\begin{tabular}{llcc}
\toprule
Variant & Mask & LR 1e-4 & LR 3e-4 \\
\midrule
LoRA & none & 80.6 $\pm$ 0.3 & 78.1 $\pm$ 1.0 \\
\method & uniform $p=0.5$ & 81.7 $\pm$ 0.2 & 81.3 $\pm$ 0.7 \\
Early-heavy profile & $p=(.75,.75,.25,.25)$ & 82.1 $\pm$ 0.2 & 79.4 $\pm$ 2.6 \\
Late-heavy profile & $p=(.25,.25,.75,.75)$ & 80.0 $\pm$ 0.5 & 77.9 $\pm$ 1.0 \\
Ramp down & $p=(.8,.6,.4,.2)$ & 80.7 $\pm$ 0.8 & 79.3 $\pm$ 0.4 \\
Ramp up & $p=(.2,.4,.6,.8)$ & 80.0 $\pm$ 0.8 & 79.0 $\pm$ 2.6 \\
Exactly two loops & random pair, $\times2$ & 81.6 $\pm$ 0.4 & 78.2 $\pm$ 1.1 \\
Exactly one loop & random loop, $\times4$ & 81.0 $\pm$ 0.5 & 77.3 $\pm$ 2.4 \\
Random prefix & loops $1..T$ & 79.9 $\pm$ 1.0 & 78.9 $\pm$ 0.1 \\
Random suffix & loops $T..4$ & 79.9 $\pm$ 1.0 & 78.0 $\pm$ 0.5 \\
Schedule $p$: 0.75 $\to$ 0 & over training & 81.0 $\pm$ 0.1 & 79.4 $\pm$ 0.2 \\
Schedule $p$: 0 $\to$ 0.75 & over training & 81.3 $\pm$ 0.4 & 81.2 $\pm$ 0.8 \\
Adaptive, more where sensitive & $p_t \propto$ sensitivity & 81.4 $\pm$ 0.4 & 79.7 $\pm$ 1.2 \\
Adaptive, less where sensitive & $p_t \propto 1/$sensitivity & 81.2 $\pm$ 0.1 & 81.5 $\pm$ 0.6 \\
Learned per-loop gates & no dropout & 79.4 $\pm$ 1.0 & 80.1 $\pm$ 0.1 \\
Learned gates + \method & uniform $p=0.5$ & 80.7 $\pm$ 0.8 & 82.8 $\pm$ 0.0 \\
\midrule
Early-heavy profile (5 seeds) & $p=(.75,.75,.25,.25)$ & 81.8 $\pm$ 0.7 & 79.3 $\pm$ 1.4 \\
Learned gates + \method (5 seeds) & uniform $p=0.5$ & 81.2 $\pm$ 0.6 & 82.0 $\pm$ 0.8 \\
\bottomrule
\end{tabular}
\end{table}

Uniform masking matches the early-heavy profile at the lower rate and exceeds it at the higher rate in the five-seed follow-up.
At the higher rate, uniform masking also achieves higher accuracy than exactly sized subsets and schedules that anneal the dropout probability to zero.
These comparisons support the uniform rule as a default across the two learning rates.
Learning per-loop gates alone also remains below uniform masking on generation accuracy.
Learned gates add trainable parameters and fall below uniform masking at $10^{-4}$, so we keep the parameter-free rule.
Independent per-loop adapters remain below Loop Dropout under this recipe, with equal-rank accuracies of 80.1 / 76.9 and equal-budget accuracies of 79.7 / 79.5 at the two rates.

\FloatBarrier

\FloatBarrier
\section{Rank and Noise Studies}
\label{app:returned_suites}

These suites use Ouro-1.4B at depth four with the MetaMath-GSM-100k recipe and the same test documents, prompts and scoring rules as the rank-16 reference.

\FloatBarrier
\subsection{Rank Dependence at a Common Learning Rate}

\begin{table}[htbp]
\centering
\footnotesize
\caption{\textbf{Paired comparisons at different shared adapter ranks.} Three seeds, LR $10^{-4}$ and $\alpha/r=2$. Accuracy is in \%; differences are in percentage points.}
\label{tab:rank_results}
\begin{tabular}{llccc}
\toprule
Rank & Benchmark & LoRA & Loop Dropout & Difference \\
\midrule
4 & GSM8K & 85.82 $\pm$ 0.20 & 86.38 $\pm$ 0.50 & $+0.56$ \\
4 & MATH-500 & 39.00 $\pm$ 1.25 & 46.07 $\pm$ 1.17 & $+7.07$ \\
64 & GSM8K & 85.44 $\pm$ 1.12 & 87.01 $\pm$ 0.49 & $+1.57$ \\
64 & MATH-500 & 37.87 $\pm$ 1.70 & 43.33 $\pm$ 1.85 & $+5.47$ \\
128 & GSM8K & 85.06 $\pm$ 0.62 & 86.48 $\pm$ 0.22 & $+1.42$ \\
128 & MATH-500 & 34.80 $\pm$ 1.00 & 41.60 $\pm$ 1.39 & $+6.80$ \\
\bottomrule
\end{tabular}
\end{table}

The MATH-500 difference is positive for each of the nine rank-by-seed pairs, and the mean GSM8K difference is positive at every rank.
\paperref{Table}{tab:mask_controls} reports the rank-16 comparison at the same learning rate.

\subsection{Answer Availability and the MATH-500 Gain}
\label{app:answer_decomposition}

\begin{table}[htbp]
\centering
\small
\caption{\textbf{Decomposing the net MATH-500 accuracy difference.} Missing-answer rates are percentages. Both net-difference components use all 500 problems as denominator and sum to the overall gain in percentage points. This decomposition does not change the official scorer.}
\label{tab:answer_decomposition}
\begin{tabular}{rccccc}
\toprule
 & \multicolumn{2}{c}{Missing answers (\%)} & \multicolumn{3}{c}{Net gain (points)} \\
\cmidrule(lr){2-3}\cmidrule(lr){4-6}
Rank & LoRA & Loop Dropout & Both extractable & Missing in either & Total \\
\midrule
4 & 8.80 & 9.60 & 6.80 & 0.27 & 7.07 \\
64 & 7.60 & 5.60 & 4.80 & 0.67 & 5.47 \\
128 & 7.80 & 5.67 & 6.00 & 0.80 & 6.80 \\
\bottomrule
\end{tabular}
\end{table}

\paperref{Table}{tab:answer_decomposition} shows that most of the gain occurs where both methods produce extractable answers.
Across the three ranks, improvements on pairs with two extractable answers account for 6.80, 4.80 and 6.00 points of the total gains.
The advantage therefore persists on problems for which both methods satisfy the answer-extraction requirement.

\FloatBarrier
\subsection{Per-Seed Measurements}

\paperref{Table}{tab:returned_seeds} reports each training seed separately for the rank sweep and for the noise controls, including a second configuration of each control trained at $2\times10^{-4}$.

\begin{table}[htbp]
\centering
\footnotesize
\caption{\textbf{Per-seed results of the additional suites.} Seeds are ordered 101 / 102 / 103. Accuracy is in \%; $c$ is the noise strength and is inapplicable to LoRA and Loop Dropout.}
\label{tab:returned_seeds}
\begin{tabular}{lrcccc}
\toprule
Method & Rank & LR & $c$ & GSM8K & MATH-500 \\
\midrule
LoRA & 4 & 1e-4 & -- & 85.90 / 85.60 / 85.97 & 38.0 / 38.6 / 40.4 \\
Loop Dropout & 4 & 1e-4 & -- & 86.13 / 86.05 / 86.96 & 45.6 / 47.4 / 45.2 \\
LoRA & 64 & 1e-4 & -- & 84.15 / 86.20 / 85.97 & 36.2 / 39.6 / 37.8 \\
Loop Dropout & 64 & 1e-4 & -- & 87.57 / 86.66 / 86.81 & 44.4 / 41.2 / 44.4 \\
LoRA & 128 & 1e-4 & -- & 84.91 / 84.53 / 85.75 & 35.8 / 33.8 / 34.8 \\
Loop Dropout & 128 & 1e-4 & -- & 86.35 / 86.35 / 86.73 & 43.2 / 40.8 / 40.8 \\
Low-rank weight noise & 16 & 1e-4 & 1 & 85.22 / 85.60 / 86.28 & 38.2 / 39.8 / 39.4 \\
Low-rank weight noise & 16 & 2e-4 & 1 & 85.90 / 84.08 / 86.81 & 37.8 / 39.0 / 35.6 \\
Parallel noise & 16 & 1e-4 & 1 & 85.29 / 85.22 / 86.35 & 36.8 / 37.6 / 39.4 \\
Parallel noise & 16 & 2e-4 & 2 & 85.06 / 86.13 / 86.13 & 38.4 / 37.4 / 36.8 \\
\bottomrule
\end{tabular}
\end{table}
\FloatBarrier

\FloatBarrier
\section{Comparisons of Adaptation Methods}
\label{app:comparison_completion}

This section gives the baseline configurations and per-seed measurements supporting \paperref{Sections}{sec:main_results}, \papernumref{sec:expanded_baselines} and~\papernumref{sec:analysis}.
On Ouro-1.4B, we also test \method with DoRA and VeRA and use full fine-tuning as an upper reference for MATH-500.
Three-seed summaries use the same three seeds for every method in their block.

\subsection{Baseline Configurations and Per-Seed Results}
\label{app:coto_seeds}

\paragraph{CoTo-on-Ouro.}
The CoTo adaptation follows the physical-layer schedule of the official implementation.
All seven adapted projections within a physical layer share a switch, and that switch remains fixed across recurrent uses and the effective optimizer batch.
The keep probability starts at 0.1 and increases to one over the first 75\% of optimizer steps, with a nonempty set of active layers and no inverse-survival scaling.
The final checkpoint is evaluated with all adapters active.
Each run uses the same 100k training examples, 3,125 updates and rank-16 parameterization as the other mathematical comparisons.
The chosen learning rate is $2\times10^{-4}$ from the five-rate grid $\{1.25,2.5,5,10,20\}\times10^{-5}$, with the choice fixed before test evaluation.
\paperref{Table}{tab:coto_seeds} gives the per-seed results.

\begin{table}[htbp]
\centering
\small
\caption{\textbf{Per-seed CoTo-on-Ouro results.} Full GSM8K and MATH-500 test sets under the MetaMath protocol; accuracy in \%. The summary uses the sample standard deviation.}
\label{tab:coto_seeds}
\begin{tabular}{lcc}
\toprule
Seed & GSM8K & MATH-500 \\
\midrule
101 & 84.84 & 37.40 \\
102 & 85.60 & 39.00 \\
103 & 86.35 & 37.00 \\
\midrule
Mean $\pm$ SD & 85.60 $\pm$ 0.76 & 37.80 $\pm$ 1.06 \\
\bottomrule
\end{tabular}
\end{table}

\paragraph{LoRA+.}
LoRA+ chooses an $A$-factor rate of $5\times10^{-5}$ on Ouro-1.4B from the halved five-candidate grid, with the fourfold multiplier for $B$.

\paragraph{LoRA Dropout.}
We evaluate LoRA Dropout~\citep{lin2024loradropout} with three training seeds on Ouro-1.4B.
The baseline uses dropout probability 0.5 and four training masks, with deterministic inference using the expected masked update and no test-time ensemble, since ensembling would multiply inference cost by the number of masks.
Its learning rate is $10^{-4}$, chosen from the same five-candidate development grid as the other regularizers.
\paperref{Table}{tab:lora_dropout_seeds} gives the per-seed results of both methods.

\begin{table}[htbp]
\centering
\small
\caption{\textbf{LoRA+ and LoRA Dropout across three training seeds.} Ouro-1.4B, rank 16, MetaMath-GSM recipe; accuracy in \%.}
\label{tab:lora_dropout_seeds}
\begin{tabular}{lcccc}
\toprule
& \multicolumn{2}{c}{LoRA+} & \multicolumn{2}{c}{LoRA Dropout} \\
\cmidrule(lr){2-3}\cmidrule(lr){4-5}
Seed & GSM8K & MATH-500 & GSM8K & MATH-500 \\
\midrule
101 & 85.29 & 37.80 & 85.44 & 42.20 \\
102 & 85.60 & 38.60 & 85.29 & 41.80 \\
103 & 85.75 & 38.00 & 86.58 & 42.40 \\
\midrule
Mean $\pm$ SD & 85.54 $\pm$ 0.23 & 38.13 $\pm$ 0.42 & 85.77 $\pm$ 0.70 & 42.13 $\pm$ 0.31 \\
\bottomrule
\end{tabular}
\end{table}

\FloatBarrier
\subsection{Adapter Sharing on Ouro-2.6B}
\label{app:ouro26_tuned}

\paperref{Table}{tab:ouro26_tuned} compares shared and independent adapters with five learning-rate candidates per method, development selection on seed 101 and three training seeds at the chosen rate.
Independent rank-four adapters match the 30.3M parameters of the shared rank-16 adapter; independent rank-sixteen adapters use 121.1M parameters.
The shared rows are the same checkpoints as in \paperref{Table}{tab:math_main}; \paperref{Table}{tab:ouro26_tuned_seeds} lists the per-seed results and chosen learning rates.

\begin{table}[htbp]
\centering
\small
\caption{\textbf{Shared and independent adapters on Ouro-2.6B.} MetaMath-GSM recipe, mean $\pm$ SD over three training seeds, in \%. All backbone loops and all trained adapter applications are active at inference.}
\label{tab:ouro26_tuned}
\begin{tabular}{lrcc}
\toprule
Method & Params & GSM8K & MATH-500 \\
\midrule
Shared LoRA, rank 16 & 30.3M & 87.62 $\pm$ 0.52 & 42.87 $\pm$ 1.17 \\
LoRA+, rank 16 & 30.3M & 87.21 $\pm$ 0.38 & 43.40 $\pm$ 0.53 \\
Independent, rank 4 & 30.3M & 85.95 $\pm$ 2.59 & 43.60 $\pm$ 0.40 \\
Independent, rank 16 & 121.1M & 85.87 $\pm$ 0.87 & 41.40 $\pm$ 1.71 \\
\method, rank 16 & 30.3M & \textbf{88.73 $\pm$ 0.31} & \textbf{47.20 $\pm$ 0.53} \\
\bottomrule
\end{tabular}
\end{table}

\begin{table}[!htbp]
\centering
\footnotesize
\caption{\textbf{Per-seed Ouro-2.6B results after independent tuning.} Seeds are ordered 101 / 102 / 103; accuracy in \%.}
\label{tab:ouro26_tuned_seeds}
\begin{tabular}{llcc}
\toprule
Method & LR & GSM8K & MATH-500 \\
\midrule
Shared LoRA, rank 16 & $2\times10^{-4}$ & 88.02 / 87.04 / 87.79 & 42.00 / 44.20 / 42.40 \\
LoRA+, rank 16 & $10^{-4}$ & 87.04 / 86.95 / 87.64 & 43.80 / 42.80 / 43.60 \\
Independent, rank 4 & $2\times10^{-4}$ & 86.88 / 87.95 / 83.02 & 43.20 / 44.00 / 43.60 \\
Independent, rank 16 & $2\times10^{-4}$ & 85.90 / 84.99 / 86.73 & 39.80 / 41.20 / 43.20 \\
\method, rank 16 & $10^{-4}$ & 89.01 / 88.78 / 88.40 & 47.80 / 47.00 / 46.80 \\
\bottomrule
\end{tabular}
\end{table}
\subsection{Compatibility with DoRA and VeRA}
\label{app:adapter_compatibility}

Loop Dropout acts on an adapter's complete weight update and can be combined with different parameterizations.
We evaluate DoRA~\citep{liu2024dora} and VeRA~\citep{kopiczko2024vera}, each with and without \method, on Ouro-1.4B under the MetaMath-GSM-100k recipe.
DoRA uses rank 16 and $\alpha=32$, with 15.65M trainable parameters, including its magnitude vectors.
VeRA uses shared frozen random factors of rank 1{,}024 and trains 0.69M scaling parameters.
For each method, \method masks the entire change from the frozen weight at each loop, with $p=0.5$ and inverse-survival rescaling.
Every backbone loop remains active, and inference uses the ordinary DoRA or VeRA update at every loop.

\paragraph{Training and selection.}
All four configurations follow the MetaMath recipe in \paperref{Table}{tab:recipes} and adapt the same seven projections per decoder layer.
We tune the learning rate over $\{1.25,2.5,5,10,20\}\times10^{-5}$ for DoRA and $\{2.5,5,10,20,40,80\}\times10^{-4}$ for VeRA, independently with and without \method.
For each configuration, seed 101 selects the rate by greedy accuracy on the 500-example development split, with ties resolved toward the smaller rate; seeds 102 and 103 then use that choice.
GSM8K and MATH-500 use the same prompts, generation limits and scoring rules as \paperref{Table}{tab:math_main}.

\begin{supplementtable}[htbp]
\centering
\small
\caption{\textbf{Loop Dropout with different adapter parameterizations.} Ouro-1.4B after MetaMath-GSM fine-tuning; accuracy in \%, mean $\pm$ SD over three seeds. Each pair has the same adapter parameterization and learning-rate search budget.}
\label{tab:adapter_compatibility}
\begin{tabular}{lcc}
\toprule
Method & GSM8K & MATH-500 \\
\midrule
DoRA & 85.39 $\pm$ 0.31 & 38.07 $\pm$ 1.62 \\
DoRA + \method & \textbf{86.73 $\pm$ 0.35} & \textbf{42.80 $\pm$ 1.56} \\
\midrule
VeRA & \textbf{86.50 $\pm$ 0.27} & 44.53 $\pm$ 1.15 \\
VeRA + \method & 86.48 $\pm$ 0.49 & \textbf{49.00 $\pm$ 0.53} \\
\bottomrule
\end{tabular}
\end{supplementtable}

\paperref{Table}{tab:adapter_compatibility} shows that \method improves MATH-500 by 4.73 points for DoRA and 4.47 points for VeRA, with gains in every seed for both parameterizations.
DoRA also gains 1.34 points on GSM8K, while VeRA retains essentially the same GSM8K accuracy.
These results extend application-level regularization to both magnitude--direction adaptation and vector-based adaptation with frozen random factors.

\FloatBarrier
\subsection{Full Fine-Tuning Reference}
\label{app:full_finetuning}

We compare parameter-efficient adaptation with updating all parameters of Ouro-1.4B under the MetaMath-GSM-100k recipe.
Full fine-tuning uses Adam with float32 master weights and bfloat16 computation, gradient checkpointing, 3{,}125 optimizer steps, effective batch size 32 and maximum sequence length 1{,}024.
The schedule uses 3\% warm-up followed by cosine decay to 10\% of the peak rate, zero weight decay and gradient clipping at 1.0.
Training data, prompts and test scoring match the adapter comparisons.

We tune the learning rate over $\{0.2,0.5,1,2,3,4,5,6,8,10,20\}\times10^{-6}$.
To provide an upper reference, we report the configuration with the highest MATH-500 accuracy in this seed-101 scan and repeat it with seeds 102 and 103.
The LoRA and \method rows use the development-selection protocol and results of \paperref{Table}{tab:math_main}.

\begin{supplementtable}[htbp]
\centering
\small
\caption{\textbf{Parameter-efficient adaptation and full fine-tuning.} Ouro-1.4B after MetaMath-GSM fine-tuning; accuracy in \%, mean $\pm$ SD over three seeds. Params counts trainable parameters. Full fine-tuning is the MATH-500 upper reference described above.}
\label{tab:full_finetuning}
\begin{tabular}{lrcc}
\toprule
Method & Params & GSM8K & MATH-500 \\
\midrule
Full fine-tuning & 1.4B & 85.52 $\pm$ 0.27 & \textbf{48.00 $\pm$ 0.53} \\
LoRA & 15.1M & 85.34 $\pm$ 0.64 & 38.87 $\pm$ 1.67 \\
\method & 15.1M & \textbf{86.53 $\pm$ 0.44} & 47.53 $\pm$ 2.61 \\
\bottomrule
\end{tabular}
\end{supplementtable}

\method reaches 47.53\% on MATH-500, within 0.47 points of the full fine-tuning reference, while training about 1\% of the model's parameters.
It also gives higher GSM8K accuracy and improves MATH-500 by 8.67 points over ordinary LoRA at the same trainable-parameter budget.

\FloatBarrier
\section{Robustness Across Ranks, Depths and Test Sets}
\label{app:robustness}

This section gives the full numerical results supporting \paperref{Section}{sec:model_generalization}.
The rank study uses the MetaMath-GSM recipe; the depth and shifted-test studies use the direct GSM8K recipe.

\subsection{Adapter Rank with Independent Tuning}
\label{app:rank_tuned}

Both methods use the same five-rate grid $\{1.25,2.5,5,10,20\}\times10^{-5}$ at each rank, with $\alpha/r=2$.
Each rank and method fixes its configuration independently before test evaluation, then uses that configuration for seeds 101--103.
\paperref{Figure}{fig:rank_overview} and \paperref{Table}{tab:rank_tuned} show positive mean differences on both benchmarks at every rank.
The rank-16 entries are the same three-seed results as in \paperref{Table}{tab:matched_expansion}.
The shared row in \paperref{Table}{tab:sharing_mask} reports the corresponding three-seed results (seeds 101--103).
\paperref{Table}{tab:rank_results} reports the rank study at $10^{-4}$.

\begin{table}[!htbp]
\centering
\small
\caption{\textbf{Independent tuning at each adapter rank.} Ouro-1.4B, three seeds per rank and method; full GSM8K and MATH-500 evaluation, mean $\pm$ SD in \%. Both methods have five learning-rate candidates at each rank.}
\label{tab:rank_tuned}
\begin{tabular}{rlcc}
\toprule
Rank & Method & GSM8K & MATH-500 \\
\midrule
4 & LoRA & 85.82 $\pm$ 0.20 & 39.00 $\pm$ 1.25 \\
4 & \method & 86.38 $\pm$ 0.50 & 46.07 $\pm$ 1.17 \\
16 & LoRA & 85.34 $\pm$ 0.64 & 38.87 $\pm$ 1.67 \\
16 & \method & 86.53 $\pm$ 0.44 & 47.53 $\pm$ 2.61 \\
64 & LoRA & 85.44 $\pm$ 1.12 & 37.87 $\pm$ 1.70 \\
64 & \method & 85.87 $\pm$ 0.48 & 40.73 $\pm$ 1.42 \\
128 & LoRA & 86.10 $\pm$ 0.29 & 39.27 $\pm$ 2.00 \\
128 & \method & 86.48 $\pm$ 0.22 & 41.60 $\pm$ 1.39 \\
\bottomrule
\end{tabular}
\end{table}

\begin{table}[!htbp]
\centering
\footnotesize
\setlength{\tabcolsep}{4pt}
\caption{\textbf{Per-seed results after independent rank-wise tuning.} Seeds are ordered 101 / 102 / 103; accuracy in \%. The same learning rate is used for all three seeds in a row.}
\label{tab:rank_tuned_seeds}
\begin{tabular}{rlccc}
\toprule
Rank & Method & LR & GSM8K & MATH-500 \\
\midrule
4 & LoRA & $10^{-4}$ & 85.90 / 85.60 / 85.97 & 38.00 / 38.60 / 40.40 \\
4 & \method & $10^{-4}$ & 86.13 / 86.05 / 86.96 & 45.60 / 47.40 / 45.20 \\
16 & LoRA & $2\times10^{-4}$ & 84.69 / 85.37 / 85.97 & 37.00 / 40.20 / 39.40 \\
16 & \method & $5\times10^{-5}$ & 87.04 / 86.35 / 86.20 & 48.40 / 49.60 / 44.60 \\
64 & LoRA & $10^{-4}$ & 84.15 / 86.20 / 85.97 & 36.20 / 39.60 / 37.80 \\
64 & \method & $2\times10^{-4}$ & 85.60 / 85.60 / 86.43 & 42.00 / 39.20 / 41.00 \\
128 & LoRA & $2.5\times10^{-5}$ & 85.90 / 85.97 / 86.43 & 37.00 / 40.80 / 40.00 \\
128 & \method & $10^{-4}$ & 86.35 / 86.35 / 86.73 & 43.20 / 40.80 / 40.80 \\
\bottomrule
\end{tabular}

\end{table}

\paragraph{Where the rank-16 transfer gain occurs.}
\paperref{Table}{tab:math_rank16_breakdown} partitions the same MATH-500 test set by subject and difficulty.
\method improves mean accuracy in all seven subjects and all five difficulty levels, with the largest level-wise difference at level 3.
The improvement therefore spans the subject groups rather than coming from a single category.

\begin{table}[!htbp]
\centering
\small
\caption{\textbf{MATH-500 breakdown for independently tuned rank-16 adapters.} Mean $\pm$ SD over three seeds, in \%; differences in percentage points. Each partition covers all 500 problems, with $n$ denoting the number of problems per group. Checkpoints are identical to those in \paperref{Table}{tab:rank_tuned}.}
\label{tab:math_rank16_breakdown}
\begin{tabular}{lrccc}
\toprule
Group & $n$ & LoRA & Loop Dropout & Difference \\
\midrule
Algebra & 124 & 58.33 $\pm$ 3.05 & 72.04 $\pm$ 3.05 & $+13.71$ \\
Counting \& Probability & 38 & 29.82 $\pm$ 9.24 & 39.47 $\pm$ 2.63 & $+9.65$ \\
Geometry & 41 & 32.52 $\pm$ 6.14 & 38.21 $\pm$ 3.73 & $+5.69$ \\
Intermediate Algebra & 97 & 19.59 $\pm$ 1.03 & 25.43 $\pm$ 6.21 & $+5.84$ \\
Number Theory & 62 & 40.32 $\pm$ 5.59 & 44.09 $\pm$ 4.06 & $+3.76$ \\
Prealgebra & 82 & 54.88 $\pm$ 4.40 & 62.60 $\pm$ 1.86 & $+7.72$ \\
Precalculus & 56 & 14.88 $\pm$ 3.72 & 25.60 $\pm$ 2.73 & $+10.71$ \\
\midrule
Level 1 & 43 & 75.97 $\pm$ 1.34 & 78.29 $\pm$ 3.55 & $+2.33$ \\
Level 2 & 90 & 60.74 $\pm$ 1.28 & 67.04 $\pm$ 2.31 & $+6.30$ \\
Level 3 & 105 & 45.40 $\pm$ 0.55 & 62.86 $\pm$ 1.90 & $+17.46$ \\
Level 4 & 128 & 34.11 $\pm$ 1.63 & 42.71 $\pm$ 5.32 & $+8.59$ \\
Level 5 & 134 & 11.69 $\pm$ 4.56 & 17.16 $\pm$ 2.69 & $+5.47$ \\
\bottomrule
\end{tabular}
\end{table}

\paragraph{Answer availability.}
We also decompose the rank-16 MATH-500 accuracy difference according to whether the official scorer extracts an answer for both methods.
Of the 8.67-point total gain, 7.87 points come from problems with two extractable answers and 0.80 from the remaining problems, using all 500 questions as the denominator for both components.
The gain on pairs with two extractable answers is positive in each seed: 9.80, 8.40 and 5.40 points for seeds 101--103.
Most of the improvement thus reflects correctness on questions where both methods satisfy the answer-extraction requirement.

\FloatBarrier

\FloatBarrier
\Needspace{7\baselineskip}
\subsection{Training and Evaluation Depth}
\label{app:depth_matrix}

Each saved adapter is evaluated at $K\in\{4,6,8\}$ without further fine-tuning, testing zero-shot generalization when evaluation extends beyond the training depth.
\paperref{Table}{tab:depth_matrix} reports the absolute accuracies underlying the training-by-evaluation comparison in \paperref{Figure}{fig:depth_sensitivity}.
\paperref{Figure}{fig:depth_accuracy} shows the evaluation-depth curves for adapters trained at four loops.

\begin{table}[H]
\centering
\small
\caption{\textbf{Accuracy across training and evaluation depths.} GSM8K recipe, Ouro-1.4B, LR $10^{-4}$ ($A$-factor rate for LoRA+), seeds 101--103; mean $\pm$ SD of strict exact match, \%. Loop Dropout has higher mean accuracy than LoRA in all nine depth pairs. Comparisons within a training depth share the training budget; different training depths use different numbers of recurrent passes per optimizer step.}
\label{tab:depth_matrix}
\begin{tabular}{llccc}
\toprule
Method & Train $K$ & Eval $K=4$ & Eval $K=6$ & Eval $K=8$ \\
\midrule
LoRA & 4 & 79.9 $\pm$ 1.0 & 75.0 $\pm$ 1.0 & 70.8 $\pm$ 1.9 \\
LoRA & 6 & 82.4 $\pm$ 1.1 & 82.3 $\pm$ 0.7 & 77.8 $\pm$ 0.7 \\
LoRA & 8 & 80.3 $\pm$ 1.4 & 81.8 $\pm$ 0.6 & 81.6 $\pm$ 0.9 \\
\method & 4 & 81.8 $\pm$ 0.5 & 78.0 $\pm$ 1.7 & 75.6 $\pm$ 1.3 \\
\method & 6 & 83.0 $\pm$ 0.3 & 83.3 $\pm$ 0.7 & 80.5 $\pm$ 0.8 \\
\method & 8 & 83.4 $\pm$ 0.5 & 83.6 $\pm$ 0.5 & 82.1 $\pm$ 0.8 \\
\midrule
LoRA+ & 4 & 79.8 $\pm$ 0.5 & 72.8 $\pm$ 2.1 & 70.2 $\pm$ 2.2 \\
rsLoRA & 4 & 79.8 $\pm$ 0.9 & 75.0 $\pm$ 1.4 & 71.1 $\pm$ 2.4 \\
\bottomrule
\end{tabular}
\end{table}

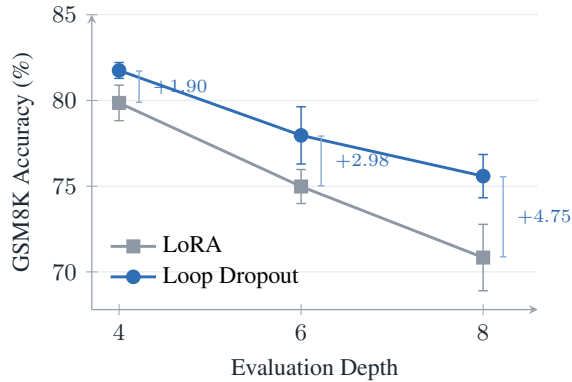
\begin{figure}[!htbp]
\centering
\begin{tikzpicture}
\begin{axis}[scale only axis,width=5.9cm,height=3.9cm,
  xmin=3.7,xmax=8.6,ymin=67.8,ymax=85,
  xtick={4,6,8},ytick={70,75,80,85},
  xlabel={Evaluation Depth},ylabel={GSM8K Accuracy (\%)},
  legend style={at={(0.02,0.04)},anchor=south west}]
\addplot+[basecurve] coordinates {(4.000000,79.858479) +- (0,1.034905) (6.000000,74.981046) +- (0,0.991410) (8.000000,70.836492) +- (0,1.936377)};
\addplot+[ourscurve] coordinates {(4.000000,81.753854) +- (0,0.463237) (6.000000,77.963103) +- (0,1.670226) (8.000000,75.587566) +- (0,1.261813)};
\legend{LoRA,Loop Dropout}
\draw[FigMain!60,line width=0.5pt,{Bar[width=2.4pt]}-{Bar[width=2.4pt]}] (axis cs:4.22,79.858479) -- (axis cs:4.22,81.753854);
\node[anchor=west,font=\scriptsize,text=FigMain] at (axis cs:4.27,80.806166) {$+1.90$};
\draw[FigMain!60,line width=0.5pt,{Bar[width=2.4pt]}-{Bar[width=2.4pt]}] (axis cs:6.22,74.981046) -- (axis cs:6.22,77.963103);
\node[anchor=west,font=\scriptsize,text=FigMain] at (axis cs:6.27,76.472075) {$+2.98$};
\draw[FigMain!60,line width=0.5pt,{Bar[width=2.4pt]}-{Bar[width=2.4pt]}] (axis cs:8.22,70.836492) -- (axis cs:8.22,75.587566);
\node[anchor=west,font=\scriptsize,text=FigMain] at (axis cs:8.27,73.212029) {$+4.75$};
\end{axis}
\end{tikzpicture}
\caption{\textbf{Zero-shot generalization beyond the training depth.} Ouro-1.4B after direct GSM8K fine-tuning at four loops, evaluated without further training; accuracy is mean $\pm$ SD over three seeds. Annotations show the gain of Loop Dropout over LoRA in percentage points.}
\label{fig:depth_accuracy}
\end{figure}

\FloatBarrier

\subsection{Learning Rates, Shifted Test Sets and Matched Depths}
\label{app:lr_shifted}

\paragraph{Learning rates.}
\paperref{Table}{tab:gsm8k_scale} compares adaptation methods on both Ouro model sizes under the direct GSM8K recipe.
\method improves mean accuracy over LoRA at every tested learning rate in \paperref{Table}{tab:lr_curve}.

\begin{table}[H]
\centering
\small
\caption{\textbf{The gain persists across learning rates and model sizes under the GSM8K recipe.} Zero-shot strict exact match (\%) after fine-tuning on the GSM8K training split, mean $\pm$ SD. The first and last blocks use three seeds; the middle block uses seeds 0--4. LR denotes the $A$-factor rate for LoRA+. Bold marks the highest mean within each model, seed set and learning rate.}
\label{tab:gsm8k_scale}
\begin{tabular}{llcc}
\toprule
Model & Method & LR & GSM8K \\
\midrule
Ouro-1.4B & LoRA & 1e-4 & 79.9 $\pm$ 1.0 \\
Ouro-1.4B & LoRA+ & 1e-4 & 79.8 $\pm$ 0.5 \\
Ouro-1.4B & rsLoRA & 1e-4 & 79.8 $\pm$ 0.9 \\
Ouro-1.4B & \method & 1e-4 & \textbf{81.8 $\pm$ 0.5} \\
\midrule
Ouro-1.4B & LoRA & 1e-4 & 80.4 $\pm$ 0.6 \\
Ouro-1.4B & \method & 1e-4 & \textbf{81.9 $\pm$ 0.4} \\
Ouro-1.4B & LoRA & 3e-4 & 78.6 $\pm$ 1.1 \\
Ouro-1.4B & \method & 3e-4 & \textbf{81.2 $\pm$ 0.6} \\
\midrule
Ouro-2.6B & LoRA & 1e-4 & 87.6 $\pm$ 0.8 \\
Ouro-2.6B & \method & 1e-4 & \textbf{89.0 $\pm$ 0.3} \\
Ouro-2.6B & LoRA & 3e-4 & 85.9 $\pm$ 0.9 \\
Ouro-2.6B & \method & 3e-4 & \textbf{88.1 $\pm$ 1.0} \\
\bottomrule
\end{tabular}
\end{table}

\begin{table}[H]
\centering
\small
\caption{\textbf{Learning-rate curve on the GSM8K recipe} (Ouro-1.4B, $K=4$; lm-eval zero-shot strict exact match, \%). Five seeds at $10^{-4}$ and $3\times10^{-4}$, three seeds elsewhere.}
\label{tab:lr_curve}
\begin{tabular}{lccc}
\toprule
LR & LoRA & \method & $\Delta$ \\
\midrule
3e-5 & 80.4 $\pm$ 0.4 & 80.9 $\pm$ 0.3 & $+0.56$ \\
1e-4 & 80.4 $\pm$ 0.6 & 81.9 $\pm$ 0.4 & $+1.52$ \\
2e-4 & 79.8 $\pm$ 1.5 & \textbf{82.2 $\pm$ 0.6} & $+2.38$ \\
3e-4 & 78.6 $\pm$ 1.1 & 81.2 $\pm$ 0.6 & $+2.62$ \\
6e-4 & 72.4 $\pm$ 0.7 & 73.8 $\pm$ 0.8 & $+1.36$ \\
\bottomrule
\end{tabular}
\end{table}

\paragraph{Shifted mathematical test sets.}
\paperref{Table}{tab:transfer} shows that saved GSM8K adapters retain positive margins on both the relabelled GSM8K-Platinum~\citep{vendrow2025platinum} and the perturbed GSM-Plus~\citep{li2024gsmplus} benchmark.
The gains on GSM-Plus mini are 1.14 and 3.40 points at the two reported learning rates, extending the comparison to perturbed problem statements.

\begin{table}[H]
\centering
\small
\caption{\textbf{Transfer of the saved GSM8K-recipe adapters} (Ouro-1.4B, three seeds) to GSM8K-Platinum (1{,}209 relabelled problems) and GSM-Plus mini (2{,}400 perturbed problems); zero-shot strict exact match, \%. ``Input dropout'' applies standard dropout of 0.1 to the LoRA input.}
\label{tab:transfer}
\footnotesize
\setlength{\tabcolsep}{3pt}
\begin{tabular}{llcccc}
\toprule
 & & \multicolumn{2}{c}{GSM8K-Platinum} & \multicolumn{2}{c}{GSM-Plus mini} \\
\cmidrule(lr){3-4}\cmidrule(lr){5-6}
Method & LR & Accuracy & $\Delta$ vs.\ LoRA & Accuracy & $\Delta$ vs.\ LoRA \\
\midrule
LoRA & 1e-4 & 82.4 $\pm$ 0.0 & -- & 58.9 $\pm$ 0.7 & -- \\
Input dropout & 1e-4 & 81.7 $\pm$ 0.4 & $-0.66$ & 59.1 $\pm$ 0.4 & $+0.21$ \\
\method & 1e-4 & \textbf{83.8 $\pm$ 0.4} & $+1.38$ & \textbf{60.0 $\pm$ 0.2} & $+1.14$ \\
LoRA & 3e-4 & 80.3 $\pm$ 0.4 & -- & 56.0 $\pm$ 1.0 & -- \\
Input dropout & 3e-4 & 80.7 $\pm$ 1.7 & $+0.33$ & 57.0 $\pm$ 1.1 & $+1.04$ \\
\method & 3e-4 & \textbf{83.1 $\pm$ 0.7} & $+2.73$ & \textbf{59.4 $\pm$ 0.4} & $+3.40$ \\
\bottomrule
\end{tabular}
\end{table}

\paragraph{Matched training and evaluation depths.}
The matched-depth comparisons in \paperref{Table}{tab:depth} show higher mean accuracy for \method across recurrence depths from two to eight loops.

\begin{table}[H]
\centering
\small
\caption{\textbf{Training and evaluating at the same recurrence depth} on the GSM8K recipe (Ouro-1.4B, LR $10^{-4}$; strict exact match, \%). Seeds 0--2 (five at $K=4$) use independent initializations; seeds 101--103 share initialization across methods.}
\label{tab:depth}
\begin{tabular}{llccc}
\toprule
$K$ & Seeds & LoRA & \method & $\Delta$ \\
\midrule
2 & 0--2 & 68.6 $\pm$ 0.8 & 70.7 $\pm$ 0.3 & $+2.17$ \\
4 & 0--4 & 80.4 $\pm$ 0.6 & 81.9 $\pm$ 0.4 & $+1.52$ \\
6 & 0--2 & 81.6 $\pm$ 0.6 & 83.2 $\pm$ 1.2 & $+1.59$ \\
\midrule
4 & 101--103 & 79.9 $\pm$ 1.0 & 81.8 $\pm$ 0.5 & $+1.90$ \\
6 & 101--103 & 82.3 $\pm$ 0.7 & 83.3 $\pm$ 0.7 & $+1.06$ \\
8 & 101--103 & 81.6 $\pm$ 0.9 & 82.1 $\pm$ 0.8 & $+0.51$ \\
\bottomrule
\end{tabular}
\end{table}

\FloatBarrier

\FloatBarrier
\section{Diagnosing and Strengthening Early-Loop Adaptation}
\label{app:single_occurrence}

\subsection{Single-Loop Activation}
\label{app:single_activation}

For a trained shared adapter, we enable its update at one loop at a time while the backbone executes all four loops.
The teacher-forced GSM8K diagnostics in this appendix each use 500 test problems.
All active gates have unit strength, and the single-loop diagnostic scores the reference continuation at the final loop.
The adapter was trained with the direct GSM8K recipe; \paperref{Table}{tab:single_occurrence} reports the two learning rates separately.
We express the improvement as a percentage reduction in loss relative to the frozen model:
\begin{equation}
 \rho_t=100\,\frac{\CE(\mathbf{0})-\CE(e_t)}{\CE(\mathbf{0})}.
\label{eq:loss_reduction}
\end{equation}
Here $e_t$ activates only loop $t$ and $\CE(\mathbf{0})=0.8445$.
All activation conditions use the same frozen-model loss as their reference, and higher values indicate stronger loss reduction.

\begin{table}[H]
\centering
\footnotesize
\setlength{\tabcolsep}{4pt}
\caption{\textbf{Single-loop activation of a trained update.} Ouro-1.4B, direct GSM8K recipe. The four loss columns are final-loop teacher-forced cross entropy (lower is better). The last column separately reports standard all-on generation accuracy on all 1,319 test questions, in \%.}
\label{tab:single_occurrence}
\begin{tabular}{llccccc}
\toprule
 & & \multicolumn{4}{c}{Final-loop loss} & All-on \\
\cmidrule(lr){3-6}
Method & LR & Only 1 & Only 2 & Only 3 & Only 4 & accuracy \\
\midrule
LoRA & $10^{-4}$ & 0.7437 & 0.6834 & 0.6208 & 0.5709 & 80.59 \\
Unscaled & $10^{-4}$ & 0.5309 & 0.5082 & 0.4984 & 0.5037 & 77.03 \\
\method & $10^{-4}$ & 0.5847 & 0.5475 & 0.5409 & 0.5432 & 81.73 \\
\midrule
LoRA & $3\times10^{-4}$ & 0.7493 & 0.6835 & 0.6159 & 0.5317 & 77.86 \\
Unscaled & $3\times10^{-4}$ & 0.5152 & 0.4907 & 0.4824 & 0.4895 & 73.84 \\
\method & $3\times10^{-4}$ & 0.5587 & 0.5302 & 0.5199 & 0.5198 & 82.11 \\
\bottomrule
\end{tabular}
\end{table}

At $10^{-4}$, the single-loop loss reductions are 11.9/19.1/26.5/32.4\% for LoRA and 30.8/35.2/35.9/35.7\% for \method.
LoRA's early applications lag far behind its final application in loss reduction.
\method lowers the raw single-loop loss at every position relative to LoRA, with the largest improvement at the first loop.
The second learning rate gives the same pattern: 11.3/19.1/27.1/37.0\% for LoRA and 33.8/37.2/38.4/38.4\% for \method.
\method thus strengthens early applications and narrows the disparity between recurrent positions.

Masking and rescaling have distinct roles in this comparison, and their ordering follows the readout scale.
All readouts use unit gates, the strength at which LoRA and the unscaled rule train each active application, whereas \method trains each active application at strength $1/q=2$; the single-loop readout therefore applies its update at half the training strength.
At all-on inference, the four unit-gate applications equal the expected total update of \method training, $\E[\sum_t g_t]=K=4$, and twice that of unscaled training, $qK=2$.
Accordingly, unscaled masking gives the lowest single-loop losses at its training strength, while inverse-survival rescaling yields the stronger all-on generation result at both learning rates.
The individual-loop diagnostic and the generation comparison therefore support learning reusable updates together with matching their training and inference scale.

\subsection{Generation and Readouts after MetaMath Fine-Tuning}
\label{app:metamath_diagnostics}

These diagnostics use the shared rank-16 MetaMath-GSM checkpoints of \paperref{Table}{tab:math_main}, with training seeds 101--103.
\paperref{Table}{tab:single_loop_generation} activates one adapter application at unit strength and generates answers on the full GSM8K test set.
The backbone executes all four loops for every activation position.
All four positions are reported separately, with mean and sample standard deviation over the training seeds.

\paperref{Table}{tab:metamath_readouts} keeps every adapter application enabled and scores the reference solutions at each loop's readout.
The loss is token-weighted cross entropy; this diagnostic uses the full MATH-500 set.
All readout positions use the same reference tokens.

\begin{table}[H]
\centering
\footnotesize
\setlength{\tabcolsep}{5pt}
\caption{\textbf{MATH-500 readouts along the fully adapted recurrence.} Ouro-1.4B after MetaMath-GSM fine-tuning. Token-weighted cross entropy, mean $\pm$ SD over three training seeds; lower is better. Every adapter application and all four backbone loops are enabled.}
\label{tab:metamath_readouts}
\begin{tabular}{lcc}
\toprule
Readout & LoRA & \method \\
\midrule
Loop 1 & 1.240 $\pm$ 0.041 & 1.047 $\pm$ 0.010 \\
Loop 2 & 0.836 $\pm$ 0.006 & 0.765 $\pm$ 0.005 \\
Loop 3 & 0.742 $\pm$ 0.009 & 0.717 $\pm$ 0.002 \\
Loop 4 & 0.737 $\pm$ 0.010 & 0.700 $\pm$ 0.005 \\
\bottomrule
\end{tabular}
\end{table}

The largest reduction occurs at the first readout, where cross entropy decreases from 1.240 to 1.047.
\paperref{Figure}{fig:early_adaptation} plots the single-application generation results and the MATH-500 readout trajectory.

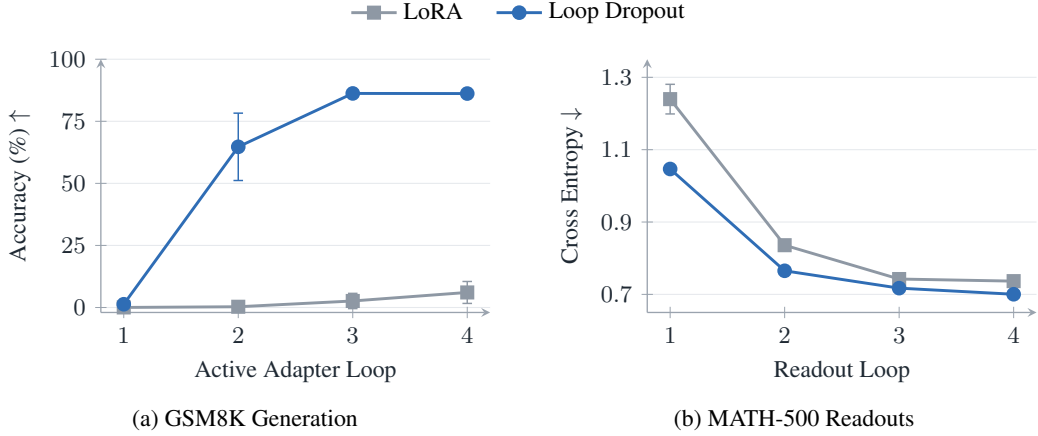
\begin{figure}[H]
\centering
\captionsetup[subfigure]{justification=centering,singlelinecheck=false}
\pgfplotslegendfromname{early-adaptation-legend}\par\vspace{0.25em}
\begin{subfigure}[t]{0.48\textwidth}
\centering
\begin{tikzpicture}
\begin{axis}[scale only axis,width=5.15cm,height=3.35cm,
  xmin=0.8,xmax=4.2,ymin=-2,ymax=100,
  xtick={1,2,3,4},ytick={0,25,50,75,100},
  xlabel={Active Adapter Loop},ylabel={Accuracy (\%) $\uparrow$},
  tick label style={font=\footnotesize,text=FigInk},
  legend to name=early-adaptation-legend,legend columns=2,
  legend style={font=\footnotesize,/tikz/every even column/.append style={column sep=1em}}]
\addplot+[basecurve] coordinates {(1,0.025272)+-(0,0.043772) (2,0.303260)+-(0,0.461165) (3,2.628254)+-(0,3.070586) (4,6.065201)+-(0,4.418786)};
\addplot+[ourscurve] coordinates {(1,1.314127)+-(0,0.687928) (2,64.720748)+-(0,13.564179) (3,86.226940)+-(0,0.838547) (4,86.176396)+-(0,0.778105)};
\legend{LoRA,Loop Dropout}
\end{axis}
\end{tikzpicture}
\caption{GSM8K Generation}
\label{fig:early_generation_gsm8k}
\end{subfigure}\hfill
\begin{subfigure}[t]{0.48\textwidth}
\centering
\begin{tikzpicture}
\begin{axis}[scale only axis,width=5.15cm,height=3.35cm,
  xmin=0.8,xmax=4.2,ymin=0.65,ymax=1.35,
  xtick={1,2,3,4},ytick={0.7,0.9,1.1,1.3},
  xlabel={Readout Loop},ylabel={Cross Entropy $\downarrow$},
  tick label style={font=\footnotesize,text=FigInk}]
\addplot+[basecurve] coordinates {(1,1.239712)+-(0,0.040942) (2,0.835959)+-(0,0.006194) (3,0.742415)+-(0,0.009203) (4,0.736633)+-(0,0.010214)};
\addplot+[ourscurve] coordinates {(1,1.046610)+-(0,0.010251) (2,0.765190)+-(0,0.004608) (3,0.717242)+-(0,0.002262) (4,0.700331)+-(0,0.004703)};
\end{axis}
\end{tikzpicture}
\caption{MATH-500 Readouts}
\label{fig:early_readout_math}
\end{subfigure}
\caption{\textbf{Generation and early readouts after MetaMath-GSM fine-tuning.} Ouro-1.4B, rank 16, four backbone loops. Left: GSM8K generation with only the indicated adapter application enabled. Right: MATH-500 teacher-forced cross entropy with all applications enabled, read out at each loop. Points and error bars show mean $\pm$ SD over three training seeds. \method supports generation from individual applications at loops two through four and lowers MATH-500 loss at every readout, most strongly at the first loop.}
\label{fig:early_adaptation}
\end{figure}

\subsection{Single-Loop Adaptation at Greater Recurrence Depths}
\label{app:depth_diagnostic}

We extend the diagnostic to Ouro-1.4B adapters trained at recurrence depths $K=4,6,8$ with the direct GSM8K recipe of \paperref{Appendix}{app:depth_matrix}.
Each depth uses shared rank-16 adapters and three matched training seeds for LoRA and \method.
Evaluation retains all $K$ backbone loops, enables one adapter application at unit strength, and scores the final-loop teacher-forced loss on the same problems as \paperref{Appendix}{app:single_activation}.
The loss reduction in \papereqref{Eq.}{eq:loss_reduction} uses the frozen-model loss at the corresponding depth as its reference.

\begin{supplementtable}[H]
\centering
\footnotesize
\setlength{\tabcolsep}{4pt}
\caption{\textbf{Single-loop adaptation at training depths four, six and eight.} Final-loop loss reduction relative to the frozen model at the same depth, in \%, mean $\pm$ SD over three training seeds. All backbone loops remain active; only the indicated adapter application is enabled. Higher is better.}
\label{tab:depth_single_loop}
\begin{tabular}{c cc cc cc}
\toprule
& \multicolumn{2}{c}{$K=4$} & \multicolumn{2}{c}{$K=6$} & \multicolumn{2}{c}{$K=8$} \\
\cmidrule(lr){2-3}\cmidrule(lr){4-5}\cmidrule(lr){6-7}
Active loop & LoRA & \method & LoRA & \method & LoRA & \method \\
\midrule
1 & 12.2 $\pm$ 0.5 & \textbf{31.1 $\pm$ 0.3} & 9.5 $\pm$ 1.1 & \textbf{23.6 $\pm$ 0.5} & 10.0 $\pm$ 0.6 & \textbf{18.3 $\pm$ 2.0} \\
2 & 19.0 $\pm$ 1.4 & \textbf{35.6 $\pm$ 0.1} & 14.4 $\pm$ 2.2 & \textbf{32.1 $\pm$ 0.4} & 14.5 $\pm$ 2.2 & \textbf{29.0 $\pm$ 1.5} \\
3 & 26.2 $\pm$ 1.4 & \textbf{36.4 $\pm$ 0.1} & 17.4 $\pm$ 3.2 & \textbf{34.1 $\pm$ 0.4} & 16.0 $\pm$ 1.8 & \textbf{31.5 $\pm$ 0.6} \\
4 & 31.7 $\pm$ 0.5 & \textbf{35.8 $\pm$ 0.0} & 15.3 $\pm$ 1.0 & \textbf{34.8 $\pm$ 0.4} & 11.4 $\pm$ 1.0 & \textbf{33.2 $\pm$ 0.6} \\
5 & -- & -- & 16.0 $\pm$ 0.9 & \textbf{34.4 $\pm$ 0.2} & 11.0 $\pm$ 0.6 & \textbf{34.4 $\pm$ 0.9} \\
6 & -- & -- & 27.7 $\pm$ 0.5 & \textbf{33.5 $\pm$ 0.2} & 12.7 $\pm$ 0.3 & \textbf{34.6 $\pm$ 0.8} \\
7 & -- & -- & -- & -- & 16.5 $\pm$ 0.9 & \textbf{34.3 $\pm$ 0.1} \\
8 & -- & -- & -- & -- & 27.9 $\pm$ 0.5 & \textbf{33.9 $\pm$ 0.5} \\
\bottomrule
\end{tabular}
\end{supplementtable}

At every depth, the final LoRA application gives the largest loss reduction in each seed.
\method improves the reduction at every position in every seed, including the earlier applications at depths six and eight.
The late-loop bias and the benefit of training across application patterns therefore persist at greater recurrence depths.

\end{document}